\documentclass{article}

\usepackage[left=2.8cm,right=2.8cm,top=2.8cm,bottom=2.8cm]{geometry}

\usepackage[utf8]{inputenc} 
\usepackage[T1]{fontenc}    
\usepackage{hyperref}       
\usepackage{url}            
\usepackage{booktabs}       
\usepackage{amsfonts}       
\usepackage{nicefrac}       
\usepackage{microtype}      
\usepackage{lipsum}		
\usepackage{graphicx}
\usepackage{natbib}
\usepackage{doi}

\usepackage{amsfonts}
\usepackage{subcaption}
\usepackage{bm}
\usepackage{algorithm}
\usepackage{algorithmic}
\usepackage{amsmath}
\usepackage{enumitem}
\usepackage{makecell}
\usepackage{multirow}

\usepackage{titlesec}
\titlespacing*{\section}{0pt}{0.75\baselineskip}{0.5\baselineskip}
\titlespacing*{\subsection}{0pt}{0.5\baselineskip}{0.25\baselineskip}
\titlespacing*{\subsubsection}{0pt}{0.25\baselineskip}{0.1\baselineskip}

\usepackage{amsthm}
\theoremstyle{plain}
\newtheorem{theorem}{Theorem}[section]

\newtheorem{lemma}[theorem]{Lemma}

\theoremstyle{definition}
\newtheorem{definition}[theorem]{Definition}

\theoremstyle{remark}

\title{Exploration by Random Reward Perturbation}

\author{Haozhe Ma$^{1}$, Guoji Fu$^{1}$, Zhengding Luo$^{2}$, Jiele Wu$^{1}$, Tze-Yun Leong$^{1}$\\[0.3cm] 
$^1$National University of Singapore\\ 
$^2$Nanyang Technological University\\[0.3cm] 
\texttt{haozhe.ma@u.nus.edu}, \texttt{guoji.fu@u.nus.edu},\texttt{luoz0021@e.ntu.edu.sg}, \\
\texttt{wujiele@comp.nus.edu.sg}, \texttt{leongty@nus.edu.sg}
}

\date{}

\hypersetup{
pdftitle={Exploration by Random Reward Perturbation},
pdfauthor={Haozhe Ma, Guoji Fu, Zhengding Luo, Jiele Wu, Tze-Yun Leong},
pdfkeywords={Reinforcement Learning, Reward Shaping, Effective Exploration},
}

\begin{document}
\maketitle

\begin{abstract}
    
    \noindent We introduce Random Reward Perturbation (RRP), a novel exploration strategy for reinforcement learning (RL). Our theoretical analyses demonstrate that adding zero-mean noise to environmental rewards effectively enhances policy diversity during training, thereby expanding the range of exploration. RRP is fully compatible with the action-perturbation-based exploration strategies, such as $\epsilon$-greedy, stochastic policies, and entropy regularization, providing additive improvements to exploration effects. It is general, lightweight, and can be integrated into existing RL algorithms with minimal implementation effort and negligible computational overhead. RRP establishes a theoretical connection between reward shaping and noise-driven exploration, highlighting their complementary potential. Experiments show that RRP significantly boosts the performance of Proximal Policy Optimization and Soft Actor-Critic, achieving higher sample efficiency and escaping local optima across various tasks, under both sparse and dense reward scenarios.

\end{abstract}

\section{Introduction}

Model-free reinforcement learning (MFRL) directly optimizes agents to maximize the expected returns without explicitly modeling the environment's dynamics. This paradigm aligns with human learning patterns in real-world tasks that are initially unfamiliar, as individuals refine their behaviors through trial and error, guided by feedback from the world. As a result, efficiently exploring diverse scenarios, actions, and outcomes becomes one of the most crucial aspects of MFRL, particularly in sparse reward environments where feedback is delayed or rare.

A wealth of research introduces randomness into agent-environment interactions to induce uncertain behaviors, such as randomly selecting actions with the $\epsilon$-greedy strategy~\citep{dqn:mnih2015human,ddpg:lillicrap2015continuous,td3:fujimoto2018addressing}, or parameterizing policies as stochastic distributions and sampling actions from them to avoid deterministic outputs~\citep{entropy:nachum2017bridging,entropy:haarnoja2017reinforcement,sac:haarnoja2018soft}. These methods are simple yet effective, providing controllable transitions from exploration to exploitation by adjusting the magnitude of randomness. However, due to the lack of directional guidance, they often struggle with extremely long-horizon tasks or extremely sparse reward settings.

Another research theme leverages reward-shaping (RS) techniques by augmenting environmental rewards with additional \textit{exploration bonuses}, also referred to as \textit{intrinsic motivations}, to explicitly encourage agents to discover novel states or actions. These shaped rewards can be derived in various ways, such as pseudo-visit-counts~\citep{count:lobel2023flipping,novelty:zhang2021noveld,novelty:ostrovski2017count,count:bellemare2016unifying}, prediction errors~\citep{diversity:badia2020never,curiosity:burda2019large,direction:fox2018dora}, novelty scores~\citep{novelty:drnd-yang2024exploration,novelty:rnd-burda2018exploration,novelty:ExploRS-devidze2022exploration}, or information gain~\citep{direction:raileanu2020ride,direction:houthooft2016vime}. While effective in sparse-reward and complex tasks, two main challenges arise. First, novelty is not inherently synonymous with value; overemphasizing it can lead agents to pursue novel yet irrelevant states, detracting from their primary learning objectives. Second, balancing shaped rewards with original rewards often relies on prior knowledge or manual tuning, which can limit adaptability.

We propose a novel and simple approach, exploration by \textbf{R}andom \textbf{R}eward \textbf{P}erturbation (\textbf{RRP}), which adopts a reward-shaping paradigm via introducing Gaussian noise into environmental rewards to enhance exploration. RRP perturbs the feedback signals received by the agent, guiding its optimization toward diverse directions, thus achieving broader exploration in a lightweight and scalable manner. The noise is gradually annealed during training to ensure a smooth transition from exploration to exploitation. Our motivation stems from two key observations: First, perturbing reward signals for exploration remains underexplored, with limited theoretical research and support, which represents a gap that RRP seeks to fill. Second, for intrinsic motivation based RS methods, we observe two practical phenomena: (1) they typically require additional models, computation overheads, or optimization processes to evaluate state novelty or prediction errors; (2) during early learning stages, when exploration is dominant, most states tend to appear equally novel to the agent's lack of experience, and prediction errors are largely random from untrained models. Motivated by these insights, we simplify the heuristic that \textit{novel states deserve higher rewards} to the more general idea that \textit{any state can be randomly rewarded}, aiming to retain competitive exploration benefits while significantly reducing computational complexity.

RRP can be incorporated into existing MFRL algorithms with minimal implementation effort, only requiring simple modifications to the reward function. We integrate RRP into two widely-used algorithms: Proximal Policy Optimization (PPO)~\citep{ppo:schulman2017proximal} and Soft Actor-Critic (SAC)~\citep{sac:haarnoja2018soft}, and evaluate their performance across multiple domains, including \textit{MuJoCo}, \textit{Mobile Manipulator}, and \textit{Dexterous Hand}, under both sparse and dense reward settings. Experiments show that RRP-PPO and RRP-SAC consistently outperform their vanilla counterparts and achieve competitive performance compared to other exploration approaches. By enhancing exploration, RRP enables agents to escape local optima in several tasks, improving sample efficiency and convergence speed.

\section{Preliminaries}

\textbf{Markov Decision Process (MDP)} is the fundamental mathematical model for formalizing RL problems, denoted by a tuple $\langle \mathcal{S}, \mathcal{A}, T, R, \gamma \rangle$, where $\mathcal{S}$ is the state space, $\mathcal{A}$ is the action space, $T(s'|s,a)$ is the transition probability from state $s$ to $s'$ given action $a$, $R(\cdot)$ is the environmental reward function, and $\gamma \in [0,1)$ is the discount factor. The objective of the agent is to learn a policy $\pi(a|s)$ that maximizes the expected cumulative reward $\mathbb{E}[\sum_{t=0}^\infty \gamma^t R(s_t, a_t)]$, where $s_t \sim T(\cdot|s_{t-1}, a_{t-1})$ and $a_t \sim \pi(\cdot|s_t)$.

\textbf{Model-Free Reinforcement Learning (MFRL)} bypasses explicit transition modeling and directly optimizes either the policy or the value function. MFRL methods are broadly categorized into two branches: \textit{On-policy} methods, such as REINFORCE~\citep{reinforce:sutton1999policy}, A2C~\citep{a3c:mnih2016asynchronous}, and PPO~\citep{ppo:schulman2017proximal}, update the policy using trajectories sampled by the current policy; and \textit{off-policy} methods, such as DQN~\citep{dqn:mnih2015human} and SAC~\citep{sac:haarnoja2018soft}, use a replay buffer to store past experiences and optimize the policy using historical data.

\textbf{Reward Shaping} is a technique that modifies the reward function to facilitate learning. It can be expressed as $R'(s,a) = R(s,a) + \Delta R(s,a)$, where $\Delta R(s,a)$ is the additional shaping reward. The goal is to design $\Delta R(s,a)$ such that it does not alter the optimal policy but accelerates learning by providing more informative feedback~\citep{lidayan2025bamdp,ibrahim2024comprehensive,durnd:ma2025catching,relara:ma2024reward,cenra:ma2024knowledge}.

\section{Methodology}

\textbf{R}andom \textbf{R}eward \textbf{P}erturbation (\textbf{RRP}) injects a noise term into the original reward signal received by the agent. Specifically, we use a zero-mean Gaussian noise $\mathcal{N}(0, \sigma^2)$, where $\sigma$ is the standard deviation and also serves as a hyperparameter to control the magnitude of the perturbation. Denote the environmental reward function as $R^{\text{env}}(s)$, the perturbed reward function is defined as:
\begin{equation}
    R^{\text{RRP}}(s) = R^{\text{env}}(s) + \varepsilon, \quad \varepsilon \sim \mathcal{N}(0, \sigma^2).
\end{equation}
By perturbing rewards, the policy is optimized toward more random directions during learning. This simple yet effective strategy allows the policy to explore diverse behaviors, uncovering more potentially promising options and thus expanding the exploration to cover a wider variety of states.

The perturbation is primarily beneficial for the exploration stage. To enable the agent to gradually refocus on the original objectives defined by the environmental reward, the noise scale is annealed over time, which is achieved by linearly decaying the standard deviation $\sigma$ as follows:
\begin{equation}
    \sigma(t) = \max \{0, \sigma_{\text{max}} - (\sigma_{\text{max}} - \sigma_{\text{min}})t/T \},
\end{equation}
where $\sigma_{\text{max}}$ and $\sigma_{\text{min}}$ are the initial and final standard deviations, $T$ is the training steps over which the noise is decayed. This scheduling mechanism ensures a smooth transition from exploration to exploitation. When $\sigma_{\text{min}} = 0$, the noise is fully eliminated, recovering the original reward signal. 

\subsection{How RRP Works}

\begin{figure*}
    \centering
    \hfill
    \begin{subfigure}[b]{0.47\textwidth}
        \includegraphics[width=\textwidth]{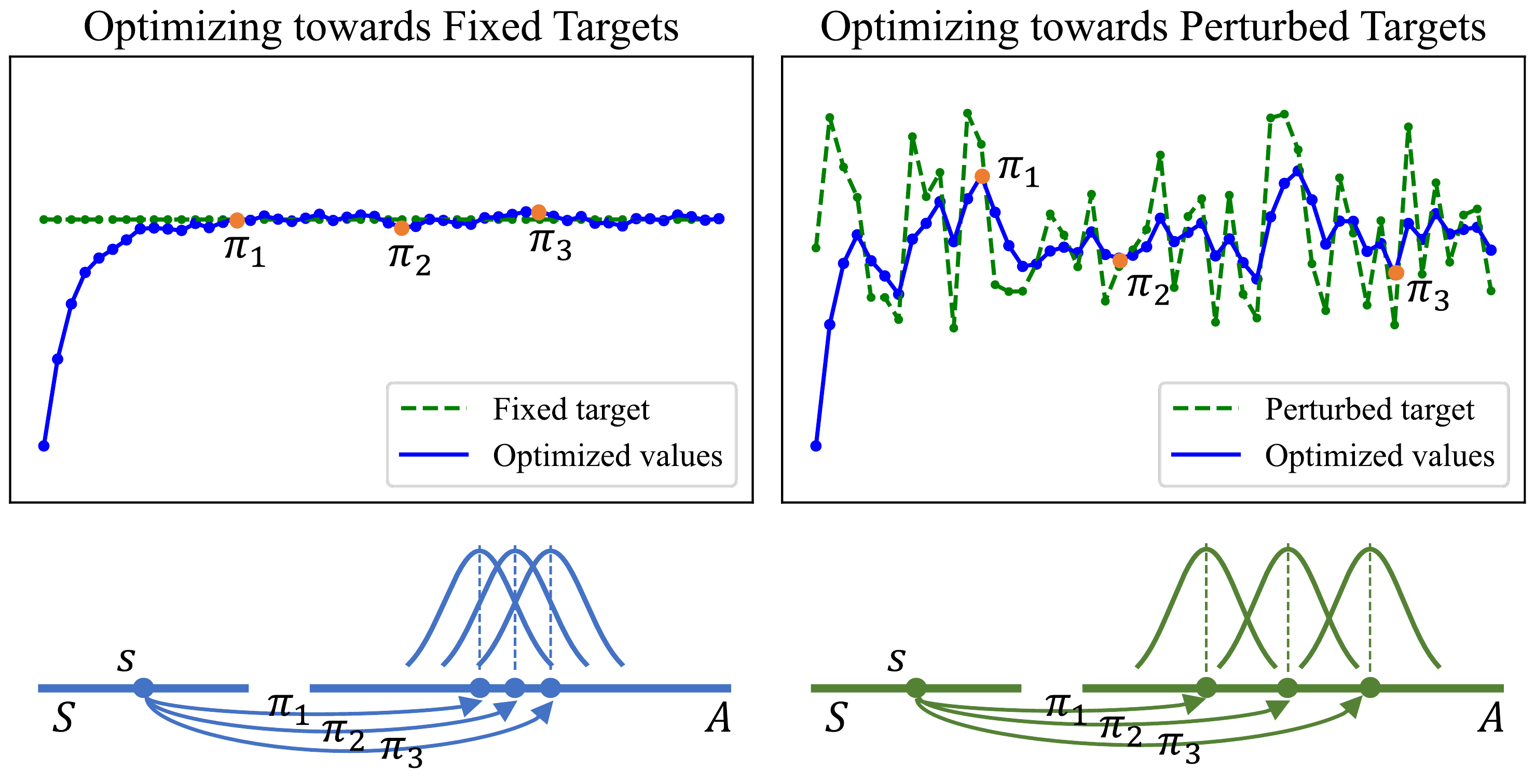}
        \caption{RRP extends policy and action diversity.}
        \label{fig:explanation-1}
    \end{subfigure}
    \hfill
    \begin{subfigure}[b]{0.52\textwidth}
        \includegraphics[width=\textwidth]{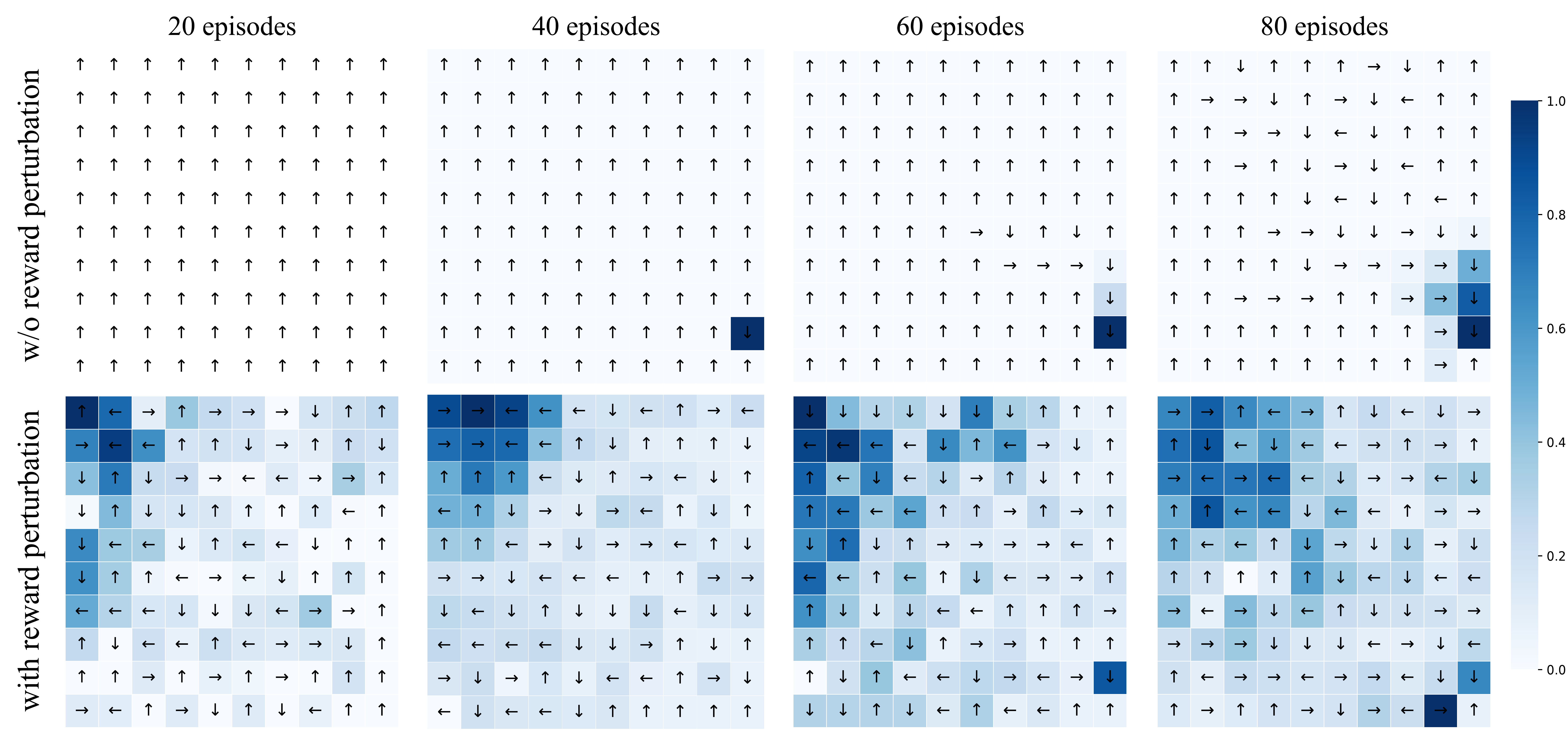}
        \caption{RRP in sparse-reward environments.}
        \label{fig:explanation-3}
    \end{subfigure}
    \hfill
    \caption{\textbf{How RRP works in exploration.} In Figure (a), by ``chasing" a perturbed target, the policy's optimization is guided toward random directions, increasing the diversity of learned policies. Combined with the effects of stochastic policies, RRP further expands the range of sampled actions, thereby improving exploration. In Figure (b), in a grid-maze environment with sparse rewards, where only the grid in the bottom-right corner is rewarded as $+1$. The top and bottom rows respectively show the values of tabular Q-learning using the original rewards and randomly perturbed rewards, and the arrows indicate the actions corresponding to the highest Q-values, visualized every $20$ training episodes. RRP clearly increases the diversity of actions selected across the grid world, while the sparse rewards trap the agent in repetitive behaviors.}
    \label{fig:explanation}
\end{figure*}

Intuitively, the effectiveness of RRP to improve exploration can be understood from three perspectives:

\textbf{Enhancing policy diversity.} The (cumulative) rewards are the objective for RL policies to maximize, either directly or indirectly~\citep{rl:sutton2018reinforcement}. Perturbing the reward signal introduces slight random shifts in the policy's optimization directions, akin to ``chasing" a moving target, which increases the diversity of policy outputs during learning. At a high level, reward perturbation inherently explores more possible policy mappings from the state space to the action space. This effect complements the commonly used stochastic policies or action-sampling exploration strategies, where uncertainty is introduced at the output level under a fixed policy. The combined effect of reward perturbation in increasing policy diversity and its synergy with action sampling is schematically illustrated in Figure~\ref{fig:explanation-1}.

\textbf{Effectiveness in sparse-reward environments.} RRP is particularly effective in sparse-reward environments, where most rewards are zero, lacking immediate and informative feedback. By perturbing the rewards, the agent can explore alternative and potentially more reasonable reward structures. For instance, in temporal-difference (TD) methods, the original TD-target is $r^{\text{env}} + \gamma V'$. In sparse-reward scenarios, where $r^{\text{env}}$ is predominantly zero, this TD-target risks driving the policy toward premature convergence in local traps. In contrast, perturbing the rewards modifies the TD-target to $r^{\text{env}} + \varepsilon + \gamma V'$. Although the noise term $\varepsilon$ carries no intrinsic meaning, it effectively prevents premature convergence. Figure~\ref{fig:explanation-3} provides an illustrative example. 

\textbf{From the reward-shaping perspective.} Reward-shaping techniques often leverage novelty bonuses to guide agents toward less frequently visited states~\citep{count:bellemare2016unifying,count:tang2017exploration,novelty:rnd-burda2018exploration}. RRP can be seen as a relaxation of the assumption that \textit{only novel states require higher rewards} to a broader notion that \textit{any state can be randomly assigned a reward}, providing two key advantages. First, in the early learning stages, the novelty differences among states are minimal, random perturbations can achieve comparable exploration outcomes. Second, RRP significantly reduces computational overhead by eliminating the need for additional calculations to evaluate state novelty, offering a simpler and more efficient alternative.

\subsection{Theoretical Analysis of Exploration via RRP}

Exploration in RL is closely tied to the variance of an agent's behaviors and state trajectories, with higher variance indicating broader exploration~\citep{survey:ladosz2022exploration}. This section presents a theoretical analysis demonstrating how reward perturbation by Gaussian noise increases the variance of policy outputs and state visits. The proofs are structured as follows: we first define and analyze the variance of model outputs in the context of stochastic gradient descent (SGD) optimization, and then conclude with theorems that apply these findings to two representative RL algorithms: Deep Q-Network (DQN)~\citep{dqn:mnih2015human} and Advantage Actor-Critic (A2C)~\citep{a3c:mnih2016asynchronous}.

\begin{definition}
\label{def:var-model-output}
Let $f_{\theta}: \mathcal{X} \to \mathbb{R}^m$ represent a function parameterized by $\theta \in \mathbb{R}^p$. Given a set of data $\{ x_i \}_{i=1}^N \in \mathcal{X}$, the \textit{variance of the model outputs} is defined as the trace of the covariance matrix of the mapping outputs:
\begin{equation}
    \setlength{\abovedisplayskip}{5pt} \setlength{\belowdisplayskip}{5pt}
    \mathcal{V}(f_{\theta}) = \operatorname{Tr}(C) \quad C = \frac{1}{N}\sum_{i=1}^N \big(f_{\theta}(x_i) - \bar{f} \big) \big(f_{\theta}(x_i) - \bar{f} \big)^T,
\end{equation}
where $\bar{f} = \frac{1}{N} \sum_{i=1}^N f_{\theta}(x_i)$.
\end{definition}
\begin{lemma}
\label{lem:noise-model-variance}
Given an initial model $f_{\theta_0}: \mathcal{X} \to \mathbb{R}^m$ and a dataset $\{(x_i, y_i)\}_{i=1}^N$, where $x_i \in \mathcal{X}$ and $y_i \in \mathbb{R}^m$ are inputs and labels. Consider two training scenarios:
\begin{enumerate}[label=(\arabic*),topsep=0pt,itemsep=2pt]
    \item Training with the original labels $y_i$.
    \item Training with Gaussian noise perturbed labels $y_i + \varepsilon_i$, where $\varepsilon_i \sim \mathcal{N}(0, \sigma^2 I_m)$.
\end{enumerate}
After one step of SGD, the following properties hold:
\begin{enumerate}[label=(\roman*),topsep=0pt,itemsep=2pt]
    \item $\mathcal{V}(f_{\theta_1}^{(2)}) > \mathcal{V}(f_{\theta_1}^{(1)})$: Adding Gaussian noise to the labels increases the variance of the model outputs.
    \item $\mathbb{E}[\bar{f}_{\theta_1}^{(2)}] = \bar{f}_{\theta_1}^{(1)}$: The expectation of the mean of the model outputs remains unchanged.
\end{enumerate}
\end{lemma}
\begin{proof}
    Let $\alpha$ be the learning rate and $\mathcal{B}$ be the mini-batch of size $B$, the loss functions under the two training scenarios are given as:
    \begin{align}
        \setlength{\abovedisplayskip}{5pt} \setlength{\belowdisplayskip}{5pt}
        L^{(1)}(\theta, \mathcal{B}) &= \frac{1}{B} \sum_{i \in \mathcal{B}} \frac{1}{2} \lVert f_\theta(x_i) - y_i \rVert^2 \\
        L^{(2)}(\theta, \mathcal{B}) &= \frac{1}{B} \sum_{i \in \mathcal{B}} \frac{1}{2} \lVert f_\theta(x_i) - (y_i + \varepsilon_i)\rVert^2.
    \end{align}
    After one step of SGD, the updated parameters under the two scenarios satisfy:    
    \begin{align}
        \setlength{\abovedisplayskip}{5pt} \setlength{\belowdisplayskip}{5pt}
        \theta_1^{(2)} &= \theta_0 - \alpha \nabla_{\theta} L^{(2)}(\theta_0, \mathcal{B}) \notag \\
        &= \theta_0 - \frac{\alpha}{B} \sum_{i \in \mathcal{B}} J_{\theta_0}(x_i)^T \big(f_{\theta_0}(x_i) - (y_i + \varepsilon_i) \big) \notag \\
        &= \theta_1^{(1)} + \frac{\alpha}{B} \sum_{i \in \mathcal{B}} J_{\theta_0}(x_i)^T \varepsilon_i,
    \end{align}
    $J_{\theta_0}(x_i)$ is the Jacobian matrix of $f_{\theta_0}(x_i)$ with respect to $\theta_0$.

    Using the \textit{first-order Taylor expansion}, the output of the model for any $x$ can be expressed as:
    \begin{equation}
        \setlength{\abovedisplayskip}{5pt} \setlength{\belowdisplayskip}{5pt}
        f_{\theta_1}(x) = f_{\theta_0}(x) + \frac{\partial f_{\theta_0}(x)}{\partial \theta} (\theta_1 - \theta_0) + o(\lVert \theta_0 - \theta_0 \rVert),
    \end{equation}
    where the higher-order term  $o(\lVert \theta_0 - \theta_0 \rVert)$ can be ignored. The outputs under the two scenarios satisfy:
    \begin{align}
    \label{eq:relation-mean-f2-f1}
        \setlength{\abovedisplayskip}{5pt} \setlength{\belowdisplayskip}{5pt}
        f_{\theta_1}^{(2)}(x) &\approx f_{\theta_0}(x) + J_{\theta_0}(x) (\theta_1^{(2)} - \theta_0) \notag \\
        &= f_{\theta_1}^{(1)}(x) + \frac{\alpha}{B} \sum_{i \in \mathcal{B}} J_{\theta_0}(x) J_{\theta_0}(x_i)^T \varepsilon_i.
    \end{align}
    Let $M(x, x_i) = J_{\theta_0}(x) J_{\theta_0}(x_i)^T$, the output means satisfy:
    \begin{equation}
        \setlength{\abovedisplayskip}{5pt} \setlength{\belowdisplayskip}{5pt}
        \bar{f}_1^{(2)} = \bar{f}_1^{(1)} + \frac{\alpha}{BN}\sum_{n=1}^N \sum_{i \in \mathcal{B}}{M(x_n,x_i)} \varepsilon_i.
    \end{equation}
    To analyze the variance, let $\Delta f_1^{(k)}(x_j) = f_{\theta_1}^{(k)}(x_j)-\bar{f}_1^{(k)}$. For the covariance matrix under scenario \textit{(2)}, we have: 
    \begingroup
    \begin{align}
        \setlength{\abovedisplayskip}{5pt} \setlength{\belowdisplayskip}{5pt}
        C^{(2)} &= \frac{1}{N} \sum_{j=1}^N{\mathbb{E}_\varepsilon\Big[ \Delta f_1^{(2)}(x_j) \big(\Delta f_1^{(2)}(x_j)\big)^T \Big]} \notag \\
        &= \frac{1}{N} \sum_{j=1}^N{\Big[\Delta f_1^{(1)}(x_j) \big(\Delta f_1^{(1)}(x_j)\big)^T + \alpha^2 B A_j \sigma^2 I_m A_j^T \Big]} \notag \\
        &= C^{(1)} + \frac{\alpha^2 B}{N} \sum_{j=1}^N{A_j \sigma^2 I_m A_j^T},
    \end{align}
    \endgroup
    where $A_j\!=\!\frac{1}{B}\sum_{i \in \mathcal{B}} M(x_j,x_i)\!-\!\frac{1}{BN}\sum_{n=1}^N \sum_{i \in \mathcal{B}} M(x_n,x_i)$.
    
    The trace of the covariance matrix satisfies:
    \begin{equation}
    \label{eq:tr-relation}
        \setlength{\abovedisplayskip}{5pt} \setlength{\belowdisplayskip}{5pt}
        \operatorname{Tr}(C^{(2)}) = \operatorname{Tr}(C^{(1)}) + \frac{\alpha^2 B\sigma^2}{N} \sum_{j=1}^N{\operatorname{Tr}(A_j A_j^T)}.
    \end{equation}
    Since $\operatorname{Tr}(A_j A_j^T) \geq 0$ and equality holds only in trivial cases, i.e., constant $M(x_j,x_i)$, we have $\mathcal{V}(f_{\theta_1}^{(2)}) > \mathcal{V}(f_{\theta_1}^{(1)})$, where conclusion \textit{(i)} holds.

    Finally, from Equation~\ref{eq:relation-mean-f2-f1}, the expectation satisfies:
    \begin{equation}
        \setlength{\abovedisplayskip}{5pt} \setlength{\belowdisplayskip}{5pt}
        \mathbb{E}_{\varepsilon} \Big[ \bar{f}_1^{(2)} \Big] = \bar{f}_1^{(1)} + \frac{\alpha}{B} \sum_{i \in \mathcal{B}} M(x_n,x_i)\mathbb{E}_{\varepsilon} \Big[ \varepsilon_i \Big]= \bar{f}_1^{(1)}.
    \end{equation}
    Therefore, conclusion \textit{(ii)} holds.

    The complete proof is provided in Appendix~\ref{app:proof-lemma1}.
\end{proof}
Lemma~\ref{lem:noise-model-variance} establishes that training with perturbed targets increases the variance of model outputs compared to those trained with noise-free targets. Notably, as indicated in Equation~\ref{eq:tr-relation}, this variance increase accumulates over successive training steps. We extend this effect to sequential cases, which aligns more closely with the RL context.
\begin{definition}
\label{def:var-trajectory}
Given $N$ trajectories $\mathcal{T}=\{(s_0, \dots, s_H)\}_i^N$ of length $H$, the \textit{variance of the trajectories} is defined as:
\begin{align}
    \setlength{\abovedisplayskip}{5pt} \setlength{\belowdisplayskip}{5pt}
    \mathcal{V}(\mathcal{T}) = \sum_{h=0}^H{\frac{1}{N}\sum_{i=1}^N{\big(s_{h,i} - \bar{s}_h\big)^2}}, \bar{s}_h = \frac{1}{N} \sum_{i=1}^N{s_{h,i}}.
\end{align}
\end{definition}
\begin{lemma}
\label{lem:var-trajectory}
Let $g^{(1)}$ and $g^{(2)}$ be two functions mapping the space $\mathcal{S} \to \mathcal{S}$, and assume their outputs variances (Definition~\ref{def:var-model-output}) satisfy $\mathcal{V}(g^{(2)}) > \mathcal{V}(g^{(1)})$. Starting from the same initial distribution $p(s_0)$, generate $N$ trajectories recursively as $s_{h+1}^{(k)} = g^{(k)}(s_h^{(k)}), k \in \{1,2\}$, and denote the sets of trajectories as $\mathcal{T}^{(1)}$ and $\mathcal{T}^{(2)}$. Then, the variance of the trajectories (Definition~\ref{def:var-trajectory}) satisfies $\mathcal{V}(\mathcal{T}^{(2)}) > \mathcal{V}(\mathcal{T}^{(1)})$.
\end{lemma}
\begin{proof}
    Starting from the initial distribution $p(s_0)$, for any step $h>0$ in the trajectory, we have:
    \begin{equation}
        \setlength{\abovedisplayskip}{5pt} \setlength{\belowdisplayskip}{5pt}
        s_h^{(k)} = p(s_0) \prod_{h=1}^H{g^{(k)}(s_{h-1}^{(k)})}, \quad k \in \{1,2\}.
    \end{equation}
    Given $\mathcal{V}(g^{(2)}) > \mathcal{V}(g^{(1)})$, for $N$ independent trajectories:
    \begin{equation}
        \setlength{\abovedisplayskip}{5pt} \setlength{\belowdisplayskip}{5pt}
        \frac{1}{N} \sum_{i=1}^N {\big(s_{h,i}^{(2)} - \bar{s}_{h}^{(2)} \big)^2} > \frac{1}{N} \sum_{i=1}^N{ \big(s_{h,i}^{(1)} - \bar{s}_{h}^{(1)} \big)^2}.
    \end{equation}
    Summing over all steps, the trajectories' variance satisfies:
    \begin{align}
        \setlength{\abovedisplayskip}{5pt} \setlength{\belowdisplayskip}{5pt}
        \mathcal{V}(\mathcal{T}^{(2)}) &= \sum_{h=0}^H {\left[\frac{1}{N} \sum_{i=1}^N (s_{h,i}^{(2)} - \bar{s}_h^{(2)})^2\right]} \notag \\ 
        &> \sum_{h=0}^H{\left[\frac{1}{N} \sum_{i=1}^N (s_{h,i}^{(1)} - \bar{s}_h^{(1)})^2\right]} = \mathcal{V}(\mathcal{T}^{(1)}).
    \end{align}

    The complete proof is provided in Appendix~\ref{app:proof-lemma2}.
\end{proof}
Lemma~\ref{lem:var-trajectory} establishes that functions with higher output variance induce trajectories with greater variance when applied sequentially. Building on this, we formalize the effect of Gaussian noise perturbation on reward functions in RL algorithms and quantify its impact on exploration.
\begin{theorem}
\label{thm:rrp-dqn}
    \textbf{RRP in DQN.} In the DQN algorithm, let $R^{\text{env}}(\cdot)$ be the environmental reward function and $R^{\text{RRP}}(\cdot) = R^{\text{env}}(\cdot) + \varepsilon$, where $\varepsilon \sim \mathcal{N}(0,\sigma^2)$ be the perturbed reward function. The variance of trajectories sampled during learning with $R^{\text{RRP}}(\cdot)$ is strictly greater than that with $R^{\text{env}}(\cdot)$, thereby expanding the exploration range.
\end{theorem}
\begin{proof}
    Let the parameterized Q-function be $Q_{\theta}: \mathcal{S} \to \mathbb{R}^{|\mathcal{A}|}$. The TD-targets for RRP-DQN and the original DQN satisfy:
    \begin{equation}
        \setlength{\abovedisplayskip}{5pt} \setlength{\belowdisplayskip}{5pt}
        y^{\text{RRP}} = R^{\text{env}}(\cdot) + \varepsilon + \gamma \max_{a'} Q_{\theta'}(s', a') = y^{\text{ori}} + \varepsilon.
    \end{equation}
    By Lemma~\ref{lem:noise-model-variance}, adding Gaussian noise to the rewards increases the variance of the Q-function outputs. The policies are derived as:
    \begin{equation}
        \setlength{\abovedisplayskip}{5pt} \setlength{\belowdisplayskip}{5pt}
        \pi^k(a|s) = \frac{\exp \big(Q_{\theta^k}(s,a) \big)}{\sum_{a'}\exp \big(Q_{\theta^k}(s,a') \big)}, k \in \{ \text{RRP}, \text{ori} \},
    \end{equation}
    Since $\mathcal{V}(Q_{\theta}^{\text{RRP}}) > \mathcal{V}(Q_{\theta}^{\text{ori}})$, it follows that $\mathcal{V}(\pi^{\text{RRP}}) > \mathcal{V}(\pi^{\text{ori}})$. Define the mapping functions:
    \begin{equation}
        \setlength{\abovedisplayskip}{5pt} \setlength{\belowdisplayskip}{5pt}
        g^k(s) = \pi^k(a|s)T(s'|s,a), k \in \{ \text{RRP}, \text{ori} \},
    \end{equation}
    where $T(s'|s,a)$ is the transition function. by Lemma~\ref{lem:var-trajectory}, the variance of the trajectories sampled under RRP-DQN is strictly greater than that under the original DQN, i.e., $\mathcal{V}(\mathcal{T}^{\text{RRP}}) > \mathcal{V}(\mathcal{T}^{\text{ori}})$, thus expanding the exploration range.

    The complete proof is provided in Appendix~\ref{app:proof-thm1}.
\end{proof}
\begin{theorem}
\label{thm:rrp-a2c}
    \textbf{RRP in A2C.} In the A2C algorithm, let $R^{\text{env}}(\cdot)$ be the environmental reward function and $R^{\text{RRP}}(\cdot) = R^{\text{env}}(\cdot) + \varepsilon$, where $\varepsilon \sim \mathcal{N}(0,\sigma^2)$ be the perturbed reward function. The variance of trajectories sampled during learning with $R^{\text{RRP}}(\cdot)$ is strictly greater than that with $R^{\text{env}}(\cdot)$, thereby expanding the exploration range.
\end{theorem}
\begin{proof}
    Let the parameterized value function be $V_{\phi}: \mathcal{S} \rightarrow \mathbb{R}$, and the policy be $\pi_{\theta}: \mathcal{S} \rightarrow \mathcal{A}$. Similar to the DQN case, the TD-targets for updating the value functions satisfy:
    \begin{equation}
        \setlength{\abovedisplayskip}{5pt} \setlength{\belowdisplayskip}{5pt}
        y^{\text{RRP}} = R^{\text{env}}(\cdot) + \varepsilon + \gamma V_{\phi'}(s') = y^{\text{ori}} + \varepsilon.
    \end{equation}
    By Lemma~\ref{lem:noise-model-variance}, adding Gaussian noise increases the variance of the value function outputs, $\mathcal{V}(V_{\phi}^{\text{RRP}}) > \mathcal{V}(V_{\phi}^{\text{ori}})$. Then consider the advantage functions:
    \begin{equation}
        \setlength{\abovedisplayskip}{5pt} \setlength{\belowdisplayskip}{5pt}
        A^k(s) = R^k(s) + \gamma V_{\phi}^k(s') - V_{\phi}^k(s), k \in \{ \text{RRP}, \text{ori} \}.
    \end{equation}
    The variances of the advantage values satisfy:
    \begin{align}
        \setlength{\abovedisplayskip}{5pt} \setlength{\belowdisplayskip}{5pt}
        \mathcal{V}[A^{\text{RRP}}(s)] - \mathcal{V}[A^{\text{ori}}(s)] = &\Big( \mathcal{V}[V^{\text{RRP}}_{\phi}(s)] - \mathcal{V}[V^{\text{ori}}_{\phi}(s)]\Big) \notag \\
        &+ \gamma^2 \Big( \mathcal{V}[V^{\text{RRP}}_{\phi}(s')] -\mathcal{V}[V^{\text{ori}}_{\phi}(s')]\Big) + \sigma^2 > 0,
    \end{align}
    which follows that $\mathcal{V}(A^{\text{RRP}}) > \mathcal{V}(A^{\text{ori}})$. By optimizing the corresponding advantages, the policies satisfy $\mathcal{V}(\pi^{\text{RRP}}) > \mathcal{V}(\pi^{\text{ori}})$. Define the mapping functions $g^k(s) = \pi^k(a|s)T(s'|s,a)$, $k \in \{ \text{RRP}, \text{ori} \}$, where \( T(s'|s, a) \) is the transition function. By Lemma~\ref{lem:var-trajectory}, we can conclude that $\mathcal{V}(\mathcal{T}^{\text{RRP}}) > \mathcal{V}(\mathcal{T}^{\text{ori}})$, thus expanding the exploration range.

    The complete proof is provided in Appendix~\ref{app:proof-thm2}.
\end{proof}

\section{Algorithms}

We integrate RRP into two representative RL algorithms: Soft Actor-Critic (SAC)~\citep{sac:haarnoja2018soft} and Proximal Policy Optimization (PPO)~\citep{ppo:schulman2017proximal}. The algorithms \textit{RRP-SAC} and \textit{RRP-PPO} are detailed in Algorithm~\ref{alg:rrp-sac} and \ref{alg:rrp-ppo}, respectively. For PPO, the annealed noise is injected into the reward signal in real-time during the interaction with the environment and directly used for policy updates. In contrast, for SAC, the unannealed noise is sampled from a Gaussian distribution during interaction and stored alongside the transition in the replay buffer. When a batch of transitions is sampled, the stored noise is annealed at this stage and used for policy optimization.

\begin{algorithm}[h]
\caption{Random Reward Perturbation Soft Actor-Critic (RRP-SAC)}
\label{alg:rrp-sac}
\begin{algorithmic}[1]
    \REQUIRE Environment $\mathcal{E}$.
    \REQUIRE Agent with policy $\pi_\theta$ and Q-function $Q_\phi$.
    \REQUIRE Augmented experience replay buffer $\mathcal{D}$.
    \REQUIRE Initial reward perturbation scale $\sigma_0$.
    \REQUIRE Noise annealing period $\lambda T$.
    
    \FOR{iteration $t = 1,2,\dots, T$}
        \FOR{each environment step}
        \STATE $a_t \sim \pi_\theta(\cdot|s_t)$ \hfill $\triangleright$ Sample action from policy.
        \STATE $s_{t+1}, r^{\text{env}}_t \sim \mathcal{E}(s_t, a_t)$  \hfill $\triangleright$ Interact with environment.
        \STATE $\varepsilon_t \sim \mathcal{N}(0, \sigma_0^2)$ \hfill $\triangleright$ sample a noise from Gaussian distribution with the \textbf{initial} scale.
        \STATE $\mathcal{D} \leftarrow \mathcal{D} \cup \{s_t, a_t, r^{\text{env}}_t, \varepsilon_t, s_{t+1}\}$  \hfill $\triangleright$ Store the augmented transition in the augmented replay buffer.
        \ENDFOR

        \FOR{each gradient step}
        \STATE $\mathcal{B} = \{s_t, a_t, r^{\text{env}}_t, \varepsilon_t, s_{t+1}\}_i^B \sim \mathcal{D}$ \hfill $\triangleright$ Sample a batch of augmented transitions.
        \STATE $\varepsilon_{t,i} \leftarrow \max \{0, \varepsilon_t - \varepsilon_t t / \lambda T \}_i$, for $i = 1,2,\dots,B$ \hfill $\triangleright$ Anneal the noise scale.
        \STATE $r^{\text{RRP}}_{t,i} = r^{\text{env}}_{t,i} + \varepsilon_{t,i}$, for $i = 1,2,\dots,B$ \hfill $\triangleright$ Perturb the reward signal.
        \STATE $\phi \leftarrow \phi - \alpha_\phi \nabla_\phi L_{\text{SAC}}(\phi;\mathcal{B})$ \hfill $\triangleright$ Update Q-function $Q_{\phi}$ using the perturbed rewards.
        \STATE $\theta \leftarrow \theta - \alpha_\theta \nabla_\theta L_{\text{SAC}}(\theta;\mathcal{B})$ \hfill $\triangleright$ Update policy $\pi_{\theta}$ using the perturbed rewards.
        \ENDFOR
    \ENDFOR
\end{algorithmic}
\end{algorithm}

\begin{algorithm}[h!]
\caption{Random Reward Perturbation Proximal Policy Optimization (RRP-PPO)}
\label{alg:rrp-ppo}

\begin{algorithmic}[1]
    \REQUIRE Environment $\mathcal{E}$.
    \REQUIRE Agent with policy $\pi_\theta$ and value function $V_\phi$.
    \REQUIRE Trajectory buffer $\mathcal{T}$ for on-policy data collection.
    \REQUIRE Initial reward perturbation scale $\sigma_0$.
    \REQUIRE Noise annealing period $\lambda T$.

\FOR{iteration $t = 1, 2, \dots, T$}
    \STATE $a_t \sim \pi_\theta(\cdot|s_t)$ \hfill $\triangleright$ Sample action from policy.
    \STATE $s_{t+1}, r^{\text{env}}_t \sim \mathcal{E}(s_t, a_t)$  \hfill $\triangleright$ Interact with environment.
    \STATE $\sigma_t = \max \{0, \sigma_0 - \sigma_0 t / \lambda T \}$ \hfill $\triangleright$ Anneal the noise scale.
    \STATE $\varepsilon_t \sim \mathcal{N}(0, \sigma_t^2)$ \hfill $\triangleright$ Sample a noise from Gaussian distribution with the \textbf{annealed} scale.
    \STATE $r^{\text{RRP}}_t = r^{\text{env}}_t + \varepsilon_t$ \hfill $\triangleright$ Perturb the reward signal.
    \STATE $\mathcal{T} \leftarrow \mathcal{T} \cup \{(s_t, a_t, r^{\text{RRP}}_t, s_{t+1})\}$ \hfill $\triangleright$ Store the transition with RRP reward in the trajectory buffer. 

    \IF{trajectory length is reached}
        \STATE $\phi \leftarrow \phi - \alpha_\phi \nabla_\phi L_{\text{PPO}}(\phi; \mathcal{T})$ \hfill $\triangleright$ Update value function $V_{\phi}$ using the perturbed rewards.
        \STATE $\theta \leftarrow \theta - \alpha_\theta \nabla_\theta L_{\text{PPO}}(\theta; \mathcal{T})$ \hfill $\triangleright$ Update policy $\pi_{\theta}$ using the perturbed rewards.
        \STATE $\mathcal{T} \leftarrow \emptyset$ \hfill $\triangleright$ Clear the trajectory buffer.
    \ENDIF
\ENDFOR
\end{algorithmic}
\end{algorithm}

\section{Experiments}

To evaluate the effectiveness of RRP, we conduct experiments across three challenging continuous-control domains: \textit{MuJoCo}~\citep{mujoco:todorov2012mujoco}, \textit{Mobile Manipulator}~\citep{robot:gymnasium_robotics2023github,robot:plappert2018multi}, and \textit{Dexterous Hand}~\citep{hand:melnik2021using}, encompassing nine tasks. For each task, we evaluate two types of environmental reward structures: \textit{dense rewards} and \textit{sparse rewards}. Dense rewards are directly tied to the agent's progress toward the goal. For example, in the \textit{RobotPush} task, the dense reward is defined as the negative distance between the block's current position and the target position; while sparse rewards are goal-conditioned, providing a reward of $1$ only when the block reaches the target, and $0$ otherwise. The sparse-reward setting is more challenging due to the absence of intermediate feedback. All tasks are shown in Figure~\ref{fig:tasks}, with detailed reward structures described in Appendix~\ref{app:tasks}.

\begin{figure}[t]
    \centering
    \includegraphics[width=\linewidth]{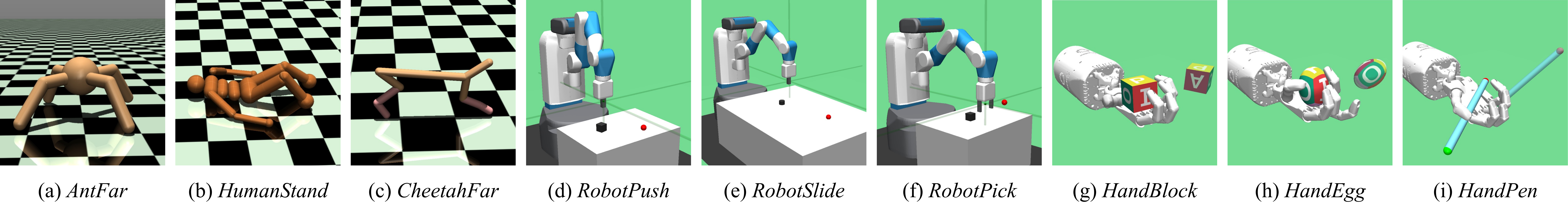}
    \caption{Continuous-control tasks in experiments with dense and sparse reward structures.}
    \label{fig:tasks}
\end{figure}

\subsection{Comparison Evaluation}

\begin{figure}[t]
    \centering
    \includegraphics[width=\linewidth]{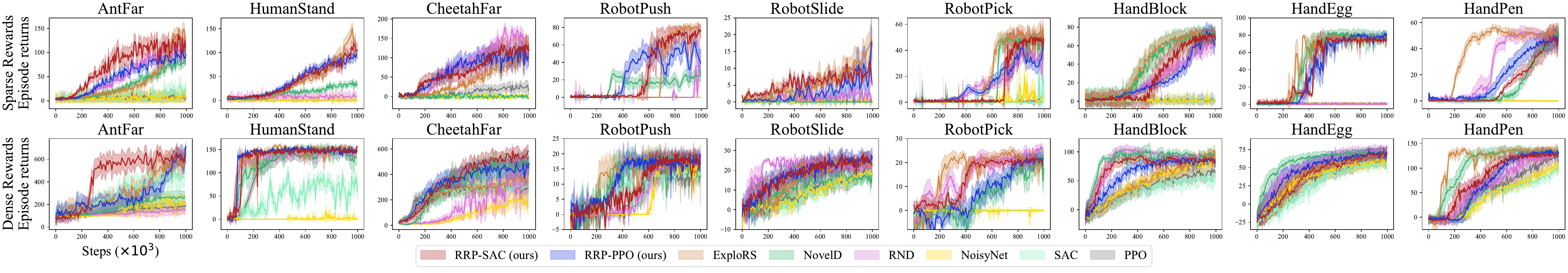}
    \caption{Comparison of RRP-SAC and RRP-PPO with several baselines on both dense-reward and sparse-reward tasks.}
    \label{fig:comparison}
\end{figure}

\begin{table}[t]
\centering
\setlength{\tabcolsep}{1.2pt}
\scriptsize
\caption{Average returns obtained from $100$ tests of the final models.}
\label{tab:model-performance}

\begin{tabular}{cc|cccc|cccc}
    \toprule
    Rewards & Tasks & RRP-SAC & RRP-PPO & SAC & PPO & ExploRS & Noveld & RND & NoisyNet \\
    \midrule
    \multirow{9}{*}{\makecell{Sparse\\Rewards}} 
    & \textit{AntFar} & 116.85$\pm$5.32 & 94.96$\pm$0.00 & 25.34$\pm$0.00 & 4.62$\pm$1.61 & 128.47$\pm$11.80 & 81.68$\pm$0.54 & 124.20$\pm$0.85 & 12.04$\pm$4.21 \\
    & \textit{HumanStand} & 104.34$\pm$0.00 & 97.05$\pm$0.00 & 0.00$\pm$0.00 & 0.02$\pm$0.01 & 128.85$\pm$3.03 & 31.56$\pm$0.00 & 7.75$\pm$0.00 & 2.91$\pm$1.90 \\
    & \textit{CheetahFar} & 124.01$\pm$4.79 & 101.35$\pm$1.46 & 4.83$\pm$0.76 & 22.41$\pm$4.75 & 120.61$\pm$0.00 & 4.86$\pm$0.42 & 115.88$\pm$0.33 & 0.51$\pm$0.64 \\
    & \textit{RobotPush} & 75.64$\pm$0.54 & 40.33$\pm$0.54 & 0.00$\pm$0.00 & 0.00$\pm$0.00 & 78.02$\pm$0.00 & 24.63$\pm$0.08 & 41.32$\pm$0.00 & 0.00$\pm$0.00 \\
    & \textit{RobotSlide} & 15.32$\pm$0.00 & 17.76$\pm$0.00 & 0.00$\pm$0.00 & 0.11$\pm$0.08 & 17.35$\pm$0.58 & 0.00$\pm$0.00 & 1.06$\pm$0.00 & 0.18$\pm$0.13 \\
    & \textit{RobotPick} & 48.47$\pm$0.59 & 45.37$\pm$0.00 & 11.01$\pm$7.79 & 0.38$\pm$0.00 & 52.07$\pm$0.00 & 40.94$\pm$0.62 & 43.16$\pm$0.00 & 3.99$\pm$0.00 \\
    & \textit{HandBlock} & 68.33$\pm$1.00 & 60.32$\pm$0.00 & 0.30$\pm$0.00 & 0.00$\pm$0.00 & 70.54$\pm$0.65 & 70.63$\pm$1.21 & 69.66$\pm$1.52 & 0.00$\pm$0.00 \\
    & \textit{HandEgg} & 74.17$\pm$1.07 & 79.26$\pm$1.31 & 0.58$\pm$0.48 & 0.87$\pm$0.44 & 83.39$\pm$0.00 & 84.45$\pm$0.00 & 0.11$\pm$0.09 & 0.00$\pm$0.00 \\
    & \textit{HandPen} & 44.41$\pm$0.00 & 47.03$\pm$1.81 & 0.03$\pm$0.02 & 0.09$\pm$0.01 & 54.21$\pm$0.00 & 52.20$\pm$1.20 & 45.50$\pm$0.39 & 0.00$\pm$0.00 \\
    \midrule
    \multirow{9}{*}{\makecell{Dense\\Rewards}} 
    & \textit{AntFar} & 637.89$\pm$0.00 & 713.12$\pm$0.00 & 441.75$\pm$0.00 & 187.20$\pm$0.00 & 672.28$\pm$0.00 & 259.54$\pm$5.87 & 149.63$\pm$3.05 & 210.51$\pm$9.72 \\
    & \textit{HumanStand} & 149.72$\pm$4.91 & 147.42$\pm$0.83 & 58.90$\pm$0.00 & 110.48$\pm$0.00 & 146.10$\pm$0.00 & 142.38$\pm$0.00 & 147.92$\pm$0.00 & 1.15$\pm$0.33 \\
    & \textit{CheetahFar} & 583.34$\pm$0.00 & 487.31$\pm$0.00 & 291.41$\pm$0.00 & 297.67$\pm$0.00 & 402.04$\pm$0.00 & 502.23$\pm$0.00 & 329.05$\pm$13.54 & 163.20$\pm$7.42 \\
    & \textit{RobotPush} & 18.11$\pm$0.76 & 17.37$\pm$1.02 & 13.40$\pm$0.56 & 12.39$\pm$0.20 & 19.87$\pm$0.15 & 17.15$\pm$0.89 & 19.07$\pm$0.39 & 14.91$\pm$0.39 \\
    & \textit{RobotSlide} & 22.57$\pm$0.00 & 26.88$\pm$0.67 & 15.68$\pm$0.00 & 16.20$\pm$0.00 & 25.29$\pm$0.00 & 20.70$\pm$0.00 & 22.76$\pm$2.19 & 17.90$\pm$0.00 \\
    & \textit{RobotPick} & 19.01$\pm$0.00 & 21.90$\pm$0.00 & 0.07$\pm$0.00 & 0.07$\pm$0.00 & 23.54$\pm$0.00 & 22.65$\pm$0.00 & 24.50$\pm$1.13 & 0.08$\pm$0.00 \\
    & \textit{HandBlock} & 80.03$\pm$0.00 & 81.97$\pm$0.00 & 55.96$\pm$0.00 & 64.58$\pm$0.00 & 89.21$\pm$0.00 & 95.67$\pm$0.00 & 88.37$\pm$0.89 & 90.28$\pm$0.00 \\
    & \textit{HandEgg} & 71.56$\pm$0.00 & 69.20$\pm$0.00 & 48.86$\pm$0.00 & 63.08$\pm$0.00 & 72.17$\pm$0.00 & 72.39$\pm$0.27 & 78.09$\pm$0.00 & 56.77$\pm$0.00 \\
    & \textit{HandPen} & 123.81$\pm$0.59 & 130.09$\pm$0.43 & 99.79$\pm$0.00 & 114.47$\pm$0.00 & 139.79$\pm$0.00 & 135.67$\pm$0.00 & 133.46$\pm$0.00 & 100.12$\pm$0.00 \\
    \bottomrule
\end{tabular}
\end{table}

We compare RRP-SAC and RRP-PPO with several well-recognized baselines across various exploration strategies. These baselines include (1) SAC~\citep{sac:haarnoja2018soft} that uses a maximum entropy term for exploration, (2) PPO~\citep{ppo:schulman2017proximal} that uses stochastic policies, (3) Random Network Distillation (RND)~\citep{novelty:rnd-burda2018exploration} that measures state novelty by prediction errors, (4) NovelD~\citep{novelty:zhang2021noveld} that accounts for novelty differences between consecutive states; (5) NoisyNet~\citep{noise-net:fortunato2018noisy} that injects noise into neural network parameters; and (6) ExploRS~\citep{novelty:ExploRS-devidze2022exploration} that combines shaped rewards from a self-supervised model and an exploration bonus. All baselines are implemented using the CleanRL library~\citep{cleanrl:huang2022cleanrl} or the official codes provided in the papers. Hyperparameters are set to their recommended values or fine-tuned to ensure a fair comparison. The average episodic returns over $5$ random seeds are shown in Figure~\ref{fig:comparison}, and the average returns obtained from $100$ tests of the final models are reported in Table~\ref{tab:model-performance}.

Compared to their respective backbone algorithms, RRP-SAC and RRP-PPO consistently deliver superior performance across all tasks in terms of both convergence speed and final returns. Figure~\ref{fig:comparison} shows that by perturbing rewards, RRP-SAC and RRP-PPO require fewer steps to converge than vanilla SAC and PPO, achieving higher sample efficiency. Additionally, they improve the ability to escape local optima, particularly in sparse-reward tasks, where SAC and PPO often converge prematurely to suboptimal returns or even fail to learn effective policies due to insufficient exploration. In contrast, RRP covers broader state spaces by perturbing the $0$-dominated sparse rewards to avoid stagnation in non-progressive optimization. This improvement is particularly evident in the \textit{Dexterous Hand} domain.

Compared to more structured exploration strategies, such as ExploRS, NovelD, and RND, the noise-driven RRP demonstrates relatively lower sample efficiency, often requiring more steps to converge. However, this performance gap is considerably reduced in dense-reward settings, where environmental rewards inherently provide informative feedback. Notably, RRP consistently achieves comparable final performance levels in most tasks, albeit with a slight delay, highlighting that RRP also effectively covers diverse states.

More importantly, RRP offers advantages in computational and resource efficiency compared to structured exploration methods, which require additional overhead to compute their exploration metrics. Specifically, RND maintains two additional neural networks and updates them with all observed states to compute novelty bonuses, ExploRS builds on RND and further incorporates a self-supervised model, while NovelD demands extra calculations to evaluate novelty differences between consecutive states. In contrast, RRP relies only on simple Gaussian noise perturbations, adding minimal time and space complexity. This makes RRP a resource-efficient alternative, particularly suitable for scenarios with constrained resources, offering a lightweight and practical solution without compromising performance.

\subsection{Effectiveness of Exploration}

\begin{figure}[t]
    \centering
    \includegraphics[width=\linewidth]{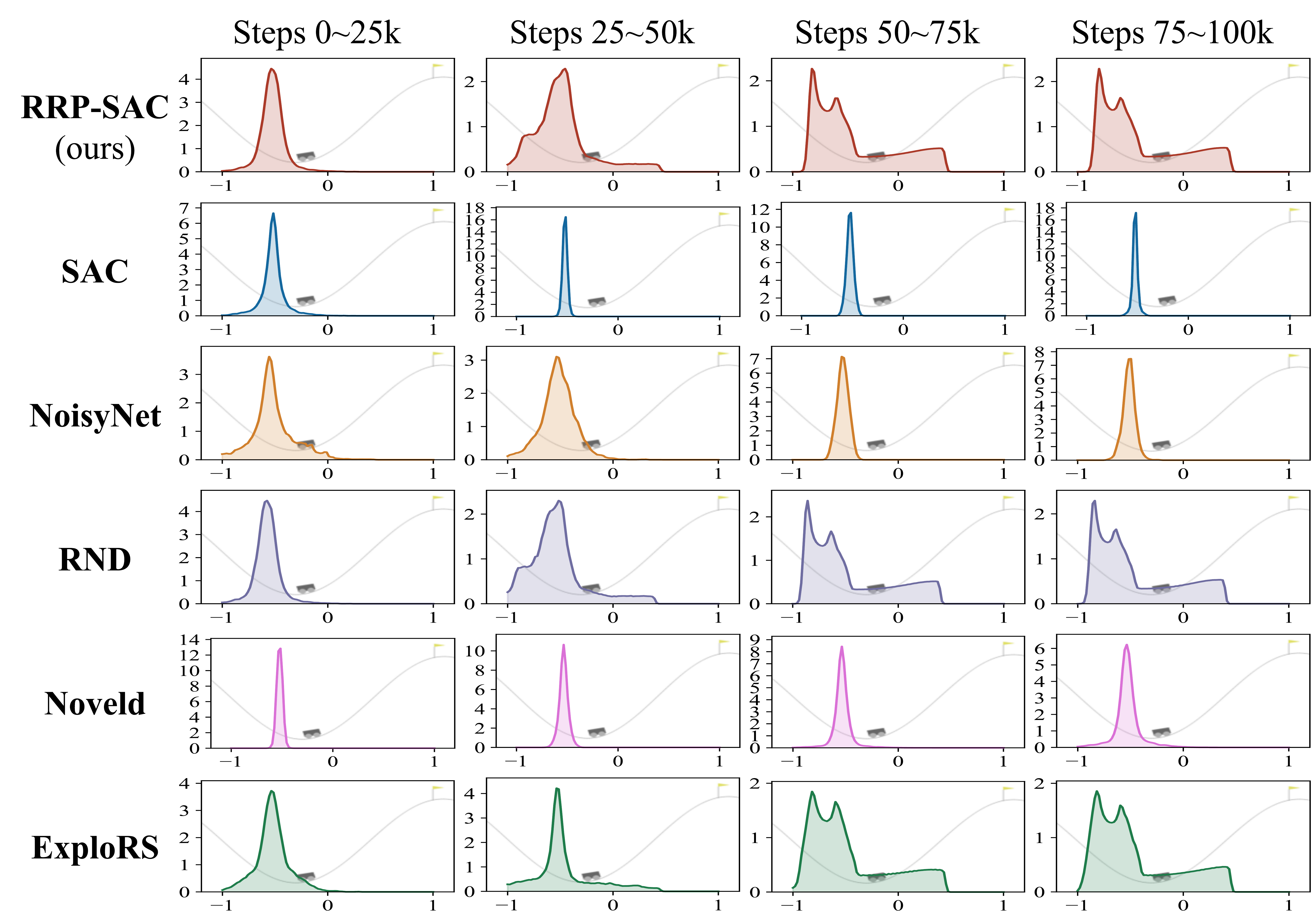}
    \caption{Density of the car's positions visited by the agents during training in the \textit{MountainCar} task.}
    \label{fig:density}
\end{figure}

To further evaluate the exploration effectiveness of RRP, we conduct a case study on the widely used \textit{MountainCar} task~\citep{mc:moore1990efficient,gymnasium:towers2024gymnasium}, where an agent controls a car by applying continuous forces to either side to slide it to the right peak of a valley. The agent receives a $+1$ reward only upon reaching the goal, while while all other states yield a negative reward proportional to the applied force, penalizing excessive force. We compare RRP-SAC with SAC, NoisyNet, RND, NovelD, and ExploRS baselines by analyzing the car's visited positions for every $25,000$ training steps. The resulting density estimated by Kernel Density Estimation (KDE) is visualized in Figure~\ref{fig:density}.

From the results, SAC and NoisyNet show limited early-stage exploration and converge prematurely to local optima by applying no force to minimize the cost penalties. NovelD shows a gradual expansion of exploration but fails to reach the goal within the given steps, indicating slower progress. In comparison, RRP-SAC effectively broadens the exploration range, building on top of its SAC backbone, enabling the agent to discover the goal state earlier and converge to the optimal policy. Similarly, RND and ExploRS quickly identify the goal state, and RRP-SAC achieves comparable performance. Considering RRP's lower computational complexity, these findings highlight its effectiveness as a practical exploration strategy.

\subsection{Ablation Study}

We investigate the impact of two critical hyperparameters of RRP: the noise scale, reflected by the initial variance $\sigma_0^2$, and the noise decay period $\lambda$, which determines the duration of noise annealing. For each parameter, we compare three different values. Learning curves are presented in Figures~\ref{fig:abl-scale} and Figure~\ref{fig:abl-period}, respectively. The final-model testing outcomes are provided in Appendix~\ref{app:add-experiments}.

\textbf{Noise Scale.} The noise scale is crucial for both exploration and stability. A small noise scale may lead to insufficient perturbation, limiting policy diversity and reducing exploration, while an overly large noise scale can cause overdeviations in policy optimization, potentially misdirecting the agent and requiring additional training steps to restore alignment with the original reward signal, causing unstable convergence. Experimental results show that RRP is robust to variations in $\sigma_0^2$, although convergence speed may vary slightly, the final policies learned are generally consistent.

\textbf{Noise Decay Period.} The noise decay period $\lambda$ mainly controls the exploration-exploitation balance, similar to the $\epsilon$-greedy strategy. A longer decay period (e.g., $\lambda = 0.5$) allows for extended exploration but may delay convergence, whereas a shorter decay period risks insufficient exploration, potentially leading to suboptimal policies. Experimental results indicate that while $\lambda$ may require fine-tuning for specific tasks, RRP is relatively insensitive to its exact value. Notably, using $\lambda = 0.3$, RRP consistently achieves strong performance across all tasks.

\begin{figure}[t]
    \centering

    \begin{subfigure}[b]{\textwidth}
    \includegraphics[width=\textwidth]{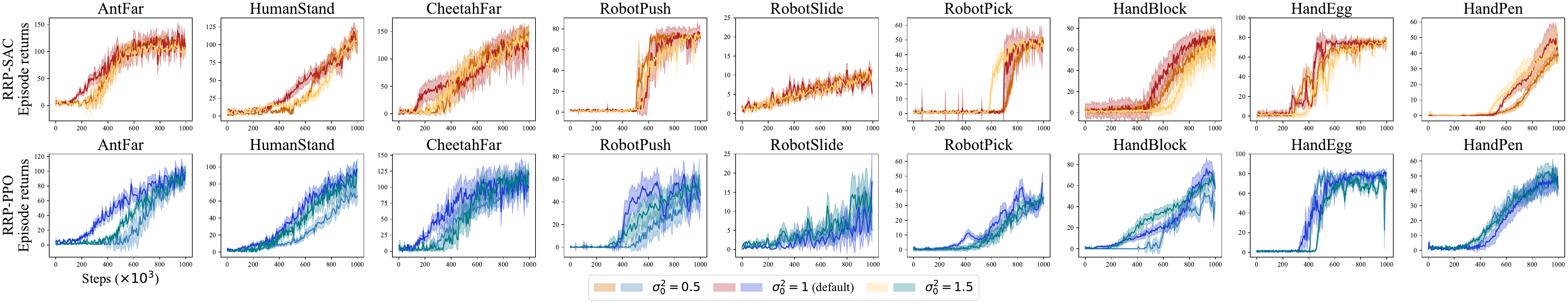}
    \caption{In the sparse-reward scenario.}
    \label{fig:abl-scale-sparse}
    \end{subfigure}

    \begin{subfigure}[b]{\textwidth}
    \includegraphics[width=\textwidth]{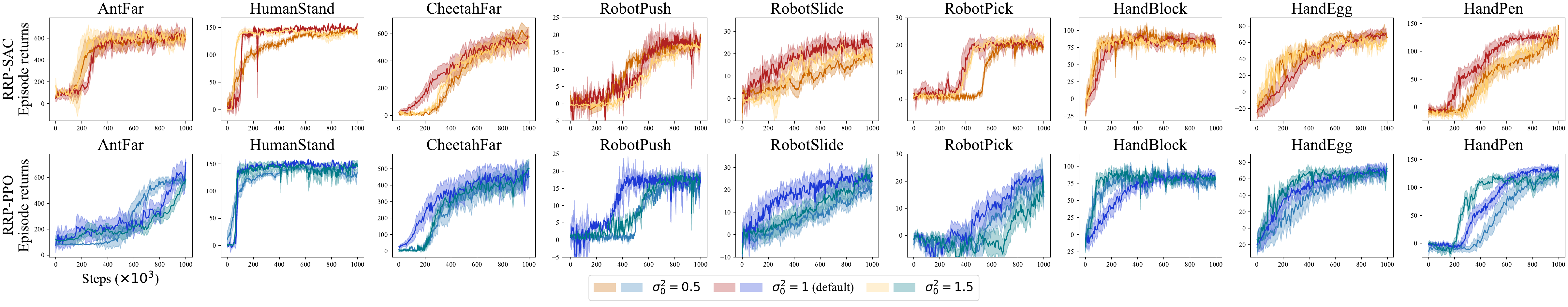}
    \caption{In the dense-reward scenario.}
    \label{fig:abl-scale-dense}
    \end{subfigure}

    \caption{Ablation study on the noise scale $\sigma_0^2$ for RRP-SAC and RRP-PPO.}
    \label{fig:abl-scale}
\end{figure}

\begin{figure}[h!]
    \centering

    \begin{subfigure}[b]{\textwidth}
    \includegraphics[width=\textwidth]{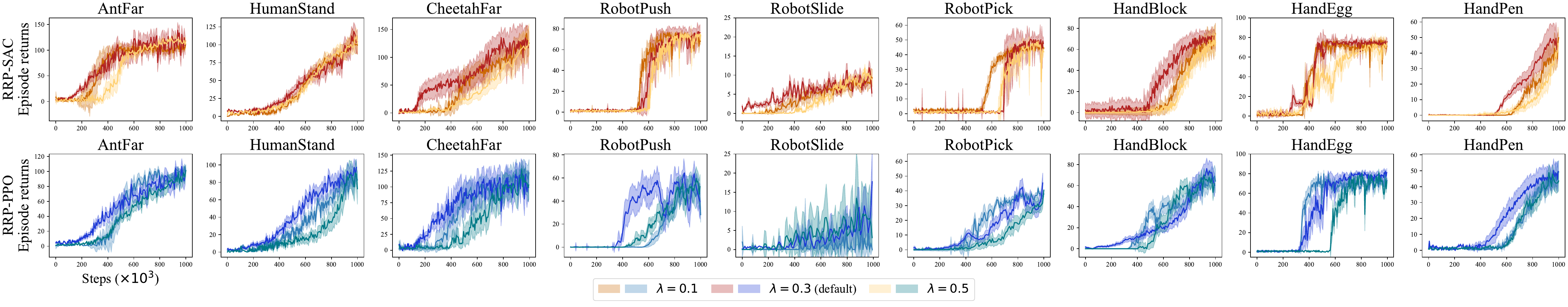}
    \caption{In the sparse-reward scenario.}
    \label{fig:abl-period-sparse}
    \end{subfigure}

    \begin{subfigure}[b]{\textwidth}
    \includegraphics[width=\textwidth]{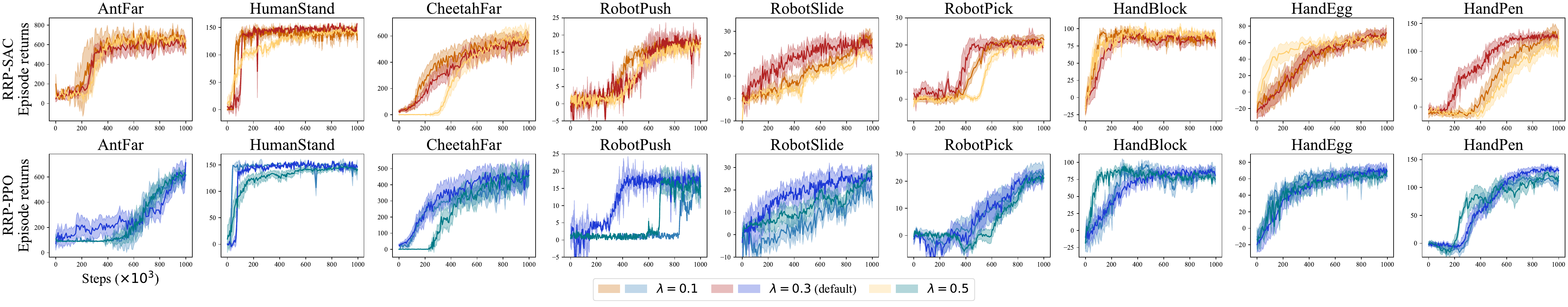}
    \caption{In the dense-reward scenario.}
    \label{fig:abl-period-dense}
    \end{subfigure}

    \caption{Ablation study on the noise decay period $\lambda$ for RRP-SAC and RRP-PPO.}
    \label{fig:abl-period}
\end{figure}

\section{Related Work}

\textbf{Random-based Exploration} introduces randomness by injecting isotropic noise to encourage exploration. A common approach is to perturb the action space, such as the well-known $\epsilon$-greedy strategy, which randomly selects actions with a decaying probability $\epsilon$~\citep{dqn:mnih2015human,a3c:mnih2016asynchronous,ddpg:lillicrap2015continuous,td3:fujimoto2018addressing}. Alternatively, policies can be represented as stochastic distributions, where actions are sampled to avoid following deterministic policies, such as Botlzmann exploration~\citep{stochastic:cesa2017boltzmann,stochastic:painter2024monte}, Thompson sampling~\citep{thompson:hu2023optimistic,thompson:agrawal2012analysis,thompson:thompson1933likelihood}, and entropy regularization~\citep{entropy:haarnoja2017reinforcement,entropy:nachum2017bridging,sac:haarnoja2018soft}. Randomness can also be introduced in other aspects of the learning process, such as scheduling learning rates~\citep{noise-learning:shani2019exploration}, randomizing the value functions~\citep{ishfaq2021randomized,hsu2024randomized,fang2025exploration}, perturbing the history experiences~\citep{kveton2019perturbed-multi,kveton2019perturbed-linear,wang2020reinforcement,sun2021reward}, or adding noise to the parameters of neural networks~\citep{noise-net:plappert2018parameter,noise-net:fortunato2018noisy,noise-net:shibata2015reinforcement}. These methods provide a simple and flexible mechanism to balance exploration and exploitation by controlling the degree of randomness.

\textbf{Novelty-Guided Exploration} directs agents toward less frequently visited states by assessing their novelty, effectively encouraging exploration of uncharted regions. A straightforward approach leverages visit counts~\citep{count:strehl2008analysis}, while many studies have extended this idea to pseudo-counts, generalizing to large or continuous state spaces~\citep{sasr:ma2024highly,count:lobel2023flipping,count:machado2020count,count:tang2017exploration,count:martin2017count,count:bellemare2016unifying}. Beyond pseudo-counts, reward-shaping techniques have incorporated intrinsic motivation as an auxiliary reward signal to complement environmental rewards. For example, Random Network Distillation (RND) and its extensions~\citep{novelty:rnd-burda2018exploration,novelty:osband2018randomized,novelty:zhang2021noveld,novelty:drnd-yang2024exploration} used the prediction error of a randomly initialized neural network as the exploration bonus. \citet{novelty:ostrovski2017count} employed PixelCNN as a density model to estimate state novelty, while \citet{novelty:ExploRS-devidze2022exploration} designed and maintained a parametric reward function. \citet{novelty:mguni2023learning} used another agent to compete against the main agent to encourage exploration. These methods have demonstrated effectiveness in sparse-reward environments and have successfully extended to high-dimensional observation spaces, such as images or raw sensory data. However, they usually face a challenge known as the ``noisy-TV" problem, where agents fixate on highly novel but meaningless regions, leading to suboptimal behaviors~\citep{noisy-tv:mavor2022stay}, thus, requiring careful tuning of the intrinsic rewards.

\textbf{Dynamics-Prediction Exploration} refers to strategies that leverage prediction errors in modeling environment dynamics to guide the agent's exploration. Curiosity-driven approaches, for instance, evaluate the agent's surprise by comparing predicted future states with actual observations, thereby encouraging further exploration in regions of high prediction error~\citep{curiosity:burda2019large,curiosity:pathak2017curiosity,luo2024gfanc}. Moreover, estimating environment dynamics has proven effective, employing techniques such as information gain~\citep{direction:houthooft2016vime}, outreach targets~
\citep{direction:fox2018dora}, and state representation learning based on inverse dynamics~\citep{direction:raileanu2020ride}. Additionally, increasing policy diversity to collect as many varied experiences as possible is another impactful approach~\citep{diversity:badia2020never,diversity:eysenbach2019diversity,diversity:hong2018diversity}. Besides, some studies incorporated hierarchical structures to learn high-level policies that guide exploration~\citep{mine:ma2024mixed,mine:ma2023hierarchical}. These methods typically require additional models to predict environmental dynamics, incurring computational overhead.

\section{Discussion and Conclusion}

We propose Random Reward Perturbation, a novel strategy that perturbs reward signals to diversify policy optimization directions, thereby enhancing exploration. This approach complements action-level stochastic methods and improves performance. At its core, RRP randomly explores the reward function space, allowing agents to attempt diverse reward structures when environmental feedback is limited, mitigating optimization traps caused by nearly all-zero sparse rewards. Moreover, RRP integrates flexibly into existing RL algorithms. Experimental results show that RRP enhances PPO and SAC, enabling them to escape local optima, and achieves comparable performance to structured exploration strategies with significantly lower resource overhead across various tasks.

As RRP relies on noise-driven exploration, it struggles in extremely sparse reward settings or extremely long-horizon tasks where directional exploration is more practical. However, RRP's generality makes it a promising extension for other exploration strategies for additive effects. For example, in reward shaping, one could design rewards that shift from random perturbations to structured and value-based rewards, leveraging RRP's exploration benefits while promoting better convergence in later stages.

\newpage
\bibliographystyle{apalike}
\bibliography{reference}

\appendix
\onecolumn

\section{Proofs}
\subsection{Proof of Lemma~\ref{lem:noise-model-variance}}
\label{app:proof-lemma1}

Given an initial model $f_{\theta_0}: \mathcal{X} \to \mathbb{R}^m$ and a dataset $\{(x_i, y_i)\}_{i=1}^N$, where $x_i \in \mathcal{X}$ and $y_i \in \mathbb{R}^m$ are inputs and labels, respectively. Consider two training scenarios:

\begin{enumerate}[label=(\arabic*),topsep=0pt,itemsep=2pt]
    \item Training with the original labels $y_i$.
    \item Training with Gaussian noise perturbed labels $y_i + \varepsilon_i$, where $\varepsilon_i \sim \mathcal{N}(0, \sigma^2 I_m)$.
\end{enumerate}

Let $\alpha$ be the learning rate and $\mathcal{B}$ be the mini-batch of size $B$, the loss functions under the two training scenarios are given by:
\begin{align}
    L^{(1)}(\theta, \mathcal{B}) &= \frac{1}{B} \sum_{i \in \mathcal{B}} \frac{1}{2} \lVert f_\theta(x_i) - y_i \rVert^2 \\
    L^{(2)}(\theta, \mathcal{B}) &= \frac{1}{B} \sum_{i \in \mathcal{B}} \frac{1}{2} \lVert f_\theta(x_i) - (y_i + \varepsilon_i)\rVert^2.
\end{align}
For the training scenario (1), the parameters are updated for one step as:
\begin{align}
    \label{eq:der-1}
    \theta_1^{(1)} &= \theta_0 - \alpha \nabla_{\theta} L^{(1)}(\theta_0, \mathcal{B}) \notag\\
    &= \theta_0 - \frac{\alpha}{B} \sum_{i \in \mathcal{B}} J_{\theta_0}(x_i)^T \big(f_{\theta_0}(x_i) - y_i \big),
\end{align}
where $J_{\theta_0}(x_i)$ is the Jacobian matrix of $f_{\theta_0}(x_i)$ with respect to $\theta_0$. Thereby, for scenario (2), the parameters are updated as:
\begin{align}
    \theta_1^{(2)} &= \theta_0 - \alpha \nabla_{\theta} L^{(2)}(\theta_0, \mathcal{B}) \notag\\
    &= \theta_0 - \frac{\alpha}{B} \sum_{i \in \mathcal{B}} J_{\theta_0}(x_i)^T \big(f_{\theta_0}(x_i) - (y_i + \varepsilon_i) \big) \notag\\
    &= \theta_0 - \frac{\alpha}{B} \sum_{i \in \mathcal{B}} J_{\theta_0}(x_i)^T \big(f_{\theta_0}(x_i) - y_i \big) + \frac{\alpha}{B} \sum_{i \in \mathcal{B}} J_{\theta_0}(x_i)^T \varepsilon_i \notag \quad \text{(By substituting Equation~\ref{eq:der-1})} \\
    &= \theta_1^{(1)} + \frac{\alpha}{B} \sum_{i \in \mathcal{B}} J_{\theta_0}(x_i)^T \varepsilon_i.
\end{align}

Using the \textit{first-order Taylor expansion}, the output of the model for any $x$ can be expressed as:
\begin{align}
    f_{\theta_1}(x) &= f_{\theta_0}(x) + \frac{\partial f_{\theta_0}(x)}{\partial \theta} (\theta_1 - \theta_0) + o(\lVert \theta_0 - \theta_0 \rVert) \notag \\
    &= f_{\theta_0}(x) + J_{\theta_0}(x)^T (\theta_1 - \theta_0) + o(\lVert \theta_0 - \theta_0 \rVert), \notag \\
    &\approx f_{\theta_0}(x) + J_{\theta_0}(x)^T (\theta_1 - \theta_0),
\end{align}
where $o(\lVert \theta_t - \theta_{t-1} \rVert)$ is the higher-order term, which can be safely ignored for small updates, typically ensured by a sufficiently small learning rate.

Therefore, the outputs of the model for scenario (1) satisfy:
\begin{align}
    \label{eq:der-2}
    f_{\theta_1}^{(1)}(x) &\approx f_{\theta_0}(x) + J_{\theta_0}(x) (\theta_1^{(1)} - \theta_0 ) \notag\\
    &= f_{\theta_0}(x) - \frac{\alpha}{B} \sum_{i \in \mathcal{B}} J_{\theta_0}(x) J_{\theta_0}(x_i)^T \big(f_{\theta_0}(x_i) - y_i \big).
\end{align}
Similarly, the outputs of the model for scenario (2) satisfy:
\begin{align}
    f_{\theta_1}^{(2)}(x) &\approx f_{\theta_0}(x) + J_{\theta_0}(x) (\theta_1^{(2)} - \theta_0) \notag\\
    &= f_{\theta_0}(x) - \frac{\alpha}{B} \sum_{i \in \mathcal{B}} J_{\theta_0}(x) J_{\theta_0}(x_i)^T \big(f_{\theta_0}(x_i) - y_i \big) + \frac{\alpha}{B} \sum_{i \in \mathcal{B}} J_{\theta_0}(x) J_{\theta_0}(x_i)^T \varepsilon_i \quad \text{(By substituting Equation~\ref{eq:der-2})} \notag\\
    &= f_{\theta_1}^{(1)}(x) + \frac{\alpha}{B} \sum_{i \in \mathcal{B}} J_{\theta_0}(x) J_{\theta_0}(x_i)^T \varepsilon_i.
\end{align}

Let $M(x, x_i) = J_{\theta_0}(x) J_{\theta_0}(x_i)^T$, then the mean of the model outputs for two scenarios satisfy:
\begin{align}
    \label{eq:der-3}
    \bar{f}_1^{(2)} &= \frac{1}{N} \sum_{n=1}^N{f_{\theta_1}^{(2)}(x_n)} \notag \\
    &= \frac{1}{N} \sum_{n=1}^N{\Big[f_{\theta_1}^{(1)}(x_n) + \frac{\alpha}{B} \sum_{i \in \mathcal{B}} M(x_n,x_i)\varepsilon_i \Big]} \notag \\ 
    &= \bar{f}_1^{(1)} + \frac{\alpha}{BN}\sum_{n=1}^N \sum_{i \in \mathcal{B}}{M(x_n,x_i)} \varepsilon_i.
\end{align}

Then the deviation of a single output $x_j$, denoted as $\Delta f_1^{(2)}(x_j)$, for two scenarios satisfy:
\begin{align}
    \Delta f_1^{(2)}(x_j) &= f_{\theta_1}^{(2)}(x_j) - \bar{f}_1^{(2)} \notag \\
    &= \Big(f_{\theta_1}^{(1)}(x_j) + \frac{\alpha}{B} \sum_{i \in \mathcal{B}}{M(x_j,x_i)\varepsilon_i} \Big) - \Big( \bar{f}_1^{(1)} + \frac{\alpha}{NB}\sum_{n=1}^N \sum_{i \in \mathcal{B}} {M(x_n,x_i)\varepsilon_i} \Big) \notag \\
    &= \big(f_{\theta_1}^{(1)}(x_j) - \bar{f}_1^{(1)} \big) + \frac{\alpha}{B} \sum_{i \in \mathcal{B}}{M(x_j,x_i)\varepsilon_i} - \frac{\alpha}{NB}\sum_{n=1}^N \sum_{i \in \mathcal{B}} {M(x_n,x_i)\varepsilon_i} \notag \\
    &= \Delta f_1^{(1)}(x_j) + \alpha \Big(\frac{1}{B}\sum_{i \in \mathcal{B}} M(x_j,x_i)\varepsilon_i - \frac{1}{NB}\sum_{n=1}^N \sum_{i \in \mathcal{B}} M(x_n,x_i)\varepsilon_i \Big).
\end{align}
Denote $A_j = \frac{1}{B}\sum_{i \in \mathcal{B}} M(x_j,x_i) - \frac{1}{BN}\sum_{n=1}^N \sum_{i \in \mathcal{B}} M(x_n,x_i)$, then:
\begin{equation}
    \Delta f_1^{(2)}(x_j) = \Delta f_1^{(1)}(x_j) + \alpha A_j \sum_{i \in \mathcal{B}} \varepsilon_i.
\end{equation}

For scenario (2), the \textit{covariance matrix} of the outputs is given by:
\begin{align}
    C^{(2)} &= \frac{1}{N} \sum_{j=1}^N{\mathbb{E}_\varepsilon\Big[ \Delta f_1^{(2)}(x_j) \big(\Delta f_1^{(2)}(x_j)\big)^T \Big]} \notag \\
    &= \frac{1}{N} \sum_{j=1}^N{\mathbb{E}_\varepsilon\Big[ \big(\Delta f_1^{(1)}(x_j) + \alpha A_j \sum_{i \in \mathcal{B}} \varepsilon_i\big) \big(\Delta f_1^{(1)}(x_j) + \alpha A_j \sum_{i \in \mathcal{B}} \varepsilon_i\big)^T \Big]} \notag \\
    &= \frac{1}{N} \sum_{j=1}^N{\Big[\Delta f_1^{(1)}(x_j) \big(\Delta f_1^{(1)}(x_j)\big)^T + \alpha^2 A_j \mathbb{E}_\varepsilon\Big[\sum_{i \in \mathcal{B}} \varepsilon_i \sum_{i' \in \mathcal{B}} \varepsilon_{i'}^T\Big] A_j^T\Big]}.
\end{align}

Since $\varepsilon_i$ are independent and identically distributed (i.i.d.), the expectation of the noise term is given by:
\begin{equation}
    \mathbb{E}_\varepsilon[\varepsilon_i \varepsilon_{i'}^T] = 
    \begin{cases}
        \sigma^2 I_m & \text{if } i = i' \\
        0 & \text{if } i \neq i'
    \end{cases}
\end{equation}
Therefore, by denoting $\mathbb{E}_\varepsilon\Big[\sum_{i \in \mathcal{B}} \varepsilon_i \sum_{i' \in \mathcal{B}} \varepsilon_{i'}^T\Big] = B\sigma^2 I_m = B\Sigma$, the relations of covariance matrix of the outputs for scenario (2) and scenario (1) can be further derived as:
\begin{align}
    C^{(2)} &= \frac{1}{N} \sum_{j=1}^N{\Delta f_1^{(1)}(x_j) \big(\Delta f_1^{(1)}(x_j)\big)^T + \alpha^2 B A_j \Sigma A_j^T} \notag \\
    &= C^{(1)} + \frac{\alpha^2 B}{N} \sum_{j=1}^N{A_j \Sigma A_j^T}.
\end{align}

Moreover, as $\operatorname{Tr}(\Sigma) = \sigma^2$, the trace of the covariance matrix satisfies:
\begin{equation}
    \operatorname{Tr}(C^{(2)}) = \operatorname{Tr}(C^{(1)}) + \frac{\alpha^2 B\sigma^2}{N} \sum_{j=1}^N{\operatorname{Tr}(A_j A_j^T)}.
\end{equation}

We note that:
\begin{equation}
    \operatorname{Tr}(A_j A_j^T) = \|A_j\|_F^2 \geq 0,
\end{equation}
where $\| \cdot \|_F$ is the Frobenius norm, and the equality holds if and only if all the elements of $A_j$ are zero, which implies that $M(x_j,x_i)$ is constant for all $j$ and $i \in \mathcal{B}$. This is extremely unlikely to happen in practice. Therefore, we can conclude that:
\begin{equation}
    \operatorname{Tr}(C^{(2)}) = \operatorname{Tr}(C^{(1)}) + \frac{\alpha^2 B\sigma^2}{N} \sum_{j=1}^N{\operatorname{Tr}(A_j A_j^T)} > \operatorname{Tr}(C^{(1)}).
\end{equation}

Given Definition~\ref{def:var-model-output}, the variance of the model outputs for two scenarios satisfy:
\begin{equation}
    \mathcal{V}(f_{\theta_1}^{(2)}) > \mathcal{V}(f_{\theta_1}^{(1)}),   
\end{equation}
where conclusion \textit{(i)} is proved. 

By Equation~\ref{eq:der-3}, the expectation of the mean of the model outputs over the noise $\varepsilon$ for scenario (2) can be derived as:
\begin{align}
    \mathbb{E}_{\varepsilon} \Big[ \bar{f}_1^{(2)} \Big] &= \mathbb{E}_{\varepsilon} \Big[ \frac{1}{N} \sum_{n=1}^N{f_{\theta_1}^{(2)}}(x_n) \Big] \notag \\
    &= \frac{1}{N} \sum_{n=1}^N{ \mathbb{E}_{\varepsilon} \Big[ f_{\theta_1}^{(1)}(x_n) + \frac{\alpha}{B} \sum_{i \in \mathcal{B}} M(x_n,x_i)\varepsilon_i \Big] } \notag \\
    &= \frac{1}{N} \sum_{n=1}^N{ f_{\theta_1}^{(1)}(x_n) + \frac{\alpha}{B} \sum_{i \in \mathcal{B}} M(x_n,x_i)\mathbb{E}_{\varepsilon} \Big[ \varepsilon_i \Big] } \notag \\
    &= \frac{1}{N} \sum_{n=1}^N{ f_{\theta_1}^{(1)}(x_n) } \notag \\
    &= \bar{f}_1^{(1)}.
\end{align}
Therefore, the expectation of the mean of model outputs over the noise $\varepsilon$ for scenario (2) equals to the mean of the model outputs for scenario (1), which proves conclusion \textit{(ii)}.

\subsection{Proof of Lemma~\ref{lem:var-trajectory}}
\label{app:proof-lemma2}

Starting from the initial distribution $p(s_0)$, for any step $h>0$ in the trajectory, the state is output by the model $g: S \to S$ where:
\begin{equation}
    s_h^{(k)} = p(s_0) \prod_{h=1}^H{g^{(k)}(s_{h-1}^{(k)})}, \quad k \in \{1,2\}.
\end{equation}
Given the condition that the variances of model outputs for two scenarios satisfy:
\begin{equation}
\mathcal{V}(g^{(2)}) > \mathcal{V}(g^{(1)}),
\end{equation}
for each step $h$, assuming mapping $N$ samples $s_{h-1,i}$ to $s_{h,i}$ by two models $g^{(1)}$ and $g^{(2)}$, respectively, the following relation holds:
\begin{equation}
    \frac{1}{N} \sum_{i=1}^N {\big(s_{h,i}^{(2)} - \bar{s}_{h}^{(2)} \big)^2} > \frac{1}{N} \sum_{i=1}^N{ \big(s_{h,i}^{(1)} - \bar{s}_{h}^{(1)} \big)^2}.
\end{equation}

Since all trajectories start from the same initial state distribution $p(s_0)$:
\begin{equation}
    \frac{1}{N} \sum_{i=1}^N{(s_{0,i}^{(2)} - \bar{s}_0^{(2)})^2} = \frac{1}{N} \sum_{i=1}^N {(s_{0,i}^{(1)} - \bar{s}_0^{(1)})^2} = 0.
\end{equation}
By Definition~\ref{def:var-trajectory} and summing over all steps, the trajectory variances satisfy:
\begin{equation}
    \mathcal{V}(\mathcal{T}^{(2)}) = \sum_{h=0}^H {\left[\frac{1}{N} \sum_{i=1}^N (s_{h,i}^{(2)} - \bar{s}_h^{(2)})^2\right]} > \sum_{h=0}^H{\left[\frac{1}{N} \sum_{i=1}^N (s_{h,i}^{(1)} - \bar{s}_h^{(1)})^2\right]} = \mathcal{V}(\mathcal{T}^{(1)}).
\end{equation}
Therefore, Lemma~\ref{lem:var-trajectory} is proved.

\subsection{Proof of Theorem~\ref{thm:rrp-dqn}}
\label{app:proof-thm1}

Given the reward functions $R^{\text{ori}}(\cdot)=R^{\text{env}}(\cdot)$ for training scenario (1), and $R^{\text{RRP}} = R^{\text{env}}(\cdot) + \varepsilon$, where $\varepsilon \sim \mathcal{N}(0, \sigma^2)$, for training scenario (2), in the Deep Q-Network (DQN) algorithm, the Temporal-Difference (TD) target are given by:
\begin{align}
    y^{\text{ori}} &= R^{\text{env}}(\cdot) + \gamma \max_{a'} Q_{\theta'}(s', a'), \\
    y^{\text{RRP}} &= R^{\text{env}}(\cdot) + \varepsilon + \gamma \max_{a'} Q_{\theta'}(s', a'), \quad \varepsilon \sim \mathcal{N}(0, \sigma^2),
\end{align}
where the $Q_{\theta'}(s', a')$ is considered as a fixed target value. Therefore, the TD-targets satisfy:
\begin{equation}
    y^{\text{RRP}} = y^{\text{ori}} + \varepsilon.
\end{equation}

By Lemma~\ref{lem:noise-model-variance}, treating the $R^{\text{env}}$ as the original label and $\varepsilon$ as the noise, training under two scenarios will induce:
\begin{equation}
    \mathcal{V}(Q^{\text{RRP}}) > \mathcal{V}(Q^{\text{ori}}).
\end{equation}
The policies for DQN are derived as:
\begin{equation}
    \label{eq:der-4}
    \setlength{\abovedisplayskip}{5pt} \setlength{\belowdisplayskip}{5pt}
    \pi^k(a|s) = \frac{\exp \big(Q_{\theta^k}(s,a) \big)}{\sum_{a'}\exp \big(Q_{\theta^k}(s,a') \big)}, \quad k \in \{ \text{RRP}, \text{ori} \}.
\end{equation}
Following Equation~\ref{eq:der-4}, the policy variances satisfy:
\begin{equation}
    \mathcal{V}(\pi^{\text{RRP}}) > \mathcal{V}(\pi^{\text{ori}}).
\end{equation}

To align with the sequential case, the mapping function $g^k: \mathcal{S} \rightarrow \mathcal{S}$ is defined as:
\begin{equation}
    \label{eq:der-5}
    g^k(s) = \pi^k(a|s)T(s'|s,a), \quad k \in \{\text{RRP}, \text{ori}\},
\end{equation}
where $T(s'|s,a)$ is the transition function. Thereby, the trajectories for two scenarios are generated by sampling from Equation~\ref{eq:der-5} state by state:
\begin{equation}
    s^k_h = g^k(s^k_{h-1}), \quad s^k_0 \sim p(s_0), \quad h \in \{1,2,\ldots,H\}, \quad k \in \{\text{RRP}, \text{ori}\}.
\end{equation}
By Lemma~\ref{lem:var-trajectory}, the trajectory variances from the two scenarios satisfy:
\begin{equation}
    \mathcal{V}(\mathcal{T}^{\text{RRP}}) > \mathcal{V}(\mathcal{T}^{\text{ori}}),
\end{equation}
which concludes the proof of Theorem~\ref{thm:rrp-dqn}.

\subsection{Proof of Theorem~\ref{thm:rrp-a2c}}
\label{app:proof-thm2}

There are two modules in the Advantage Actor-Critic (A2C) algorithm: a policy function and a value function. Consider two training scenarios with distinct reward functions: $R^{\text{ori}}(\cdot) = R^{\text{env}}(\cdot)$ for scenario \textit{ori}, and $R^{\text{RRP}} = R^{\text{env}}(\cdot) + \varepsilon$ for scenario \textit{RRP}, where $\varepsilon \sim \mathcal{N}(0, \sigma^2)$. The policy and value functions for the two scenarios are denoted as $\pi^{\text{ori}}_\theta$ and $V^{\text{ori}}_\phi$, and $\pi^{\text{RRP}}_\theta$ and $V^{\text{RRP}}_\phi$, respectively. We define an A2C update step as consisting of two consecutive sub-steps: a value function update and a policy function update. 

We begin by considering the value function update. The loss function of the value function is given by:
\begin{equation}
    L_V(\phi) = \frac{1}{2} \lVert V_{\phi}(s) - y\rVert^2,
\end{equation}
where the TD-targets $y$ for the two scenarios are:
\begin{align}
    y^{\text{ori}} &= R^{\text{env}}(s) + \gamma V_{\phi}(s'), \\
    y^{\text{RRP}} &= R^{\text{env}}(s) + \varepsilon + \gamma V_{\phi}(s'), \quad \varepsilon \sim \mathcal{N}(0, \sigma^2).
\end{align}
in which the $V_{\phi}(s')$ is considered as a fixed target value. Therefore, by Lemma~\ref{lem:noise-model-variance}, the value function variances under two training scenarios satisfy:
\begin{equation}
    \label{eq:der-6}
    \mathcal{V}(V^{\text{RRP}}_\phi) > \mathcal{V}(V^{\text{ori}}_\phi).
\end{equation}

Now consider the policy function update. The loss function of the policy function is given by:
\begin{equation}
    L_{\pi}(\theta) = -[\log(\pi_{\theta}(a|s))A(s, a)],
\end{equation}
where $A(s, a)$ is the advantage function, and under the two scenarios, are given by:
\begin{align}
    A^{\text{ori}}(s, a) &= R^{\text{env}}(s) + \gamma V_{\phi}(s') - V_{\phi}(s), \\
    A^{\text{RRP}}(s, a) &= R^{\text{env}}(s) + \varepsilon + \gamma V_{\phi}(s') - V_{\phi}(s), \quad \varepsilon \sim \mathcal{N}(0, \sigma^2).
\end{align}

Then the variance of $A^{\text{ori}}(s, a)$ is:
\begin{align}
    \text{Var}[A^{\text{ori}}(s, a)] &= \text{Var}[R^{\text{env}}(s) + \gamma V^{\text{ori}}_{\phi}(s') - V^{\text{ori}}_{\phi}(s)] \notag\\
    &= \text{Var}[R^{\text{env}}(s)] + \gamma^2 \text{Var}[V^{\text{ori}}_{\phi}(s')] + \text{Var}[V^{\text{ori}}_{\phi}(s)],
\end{align}
while similarly, the variance of $A^{\text{RRP}}(s, a)$ is:
\begin{equation}
    \text{Var}[A^{\text{RRP}}(s, a)] = \text{Var}[R^{\text{env}}(s)] + \sigma^2 + \gamma^2 \text{Var}[V^{\text{RRP}}_{\phi}(s')] + \text{Var}[V^{\text{RRP}}_{\phi}(s)].
\end{equation}
It can be derived that:
\begin{equation}
    \label{eq:der-7}
    \text{Var}[A^{\text{RRP}}(s, a)] - \text{Var}[A^{\text{ori}}(s, a)] = \sigma^2 + \gamma^2 \Big( \text{Var}[V^{\text{RRP}}_{\phi}(s')] -\text{Var}[V^{\text{ori}}_{\phi}(s')]\Big) + \Big( \text{Var}[V^{\text{RRP}}_{\phi}(s)] - \text{Var}[V^{\text{ori}}_{\phi}(s)]\Big).
\end{equation}
It is proved that the variance of the value function outputs is higher in scenario \textit{RRP} than in scenario \textit{ori} by Equation~\ref{eq:der-6}. Therefore, $\text{Var}[V^{\text{RRP}}_{\phi}(s')] -\text{Var}[V^{\text{ori}}_{\phi}(s')] > 0$ and $\text{Var}[V^{\text{RRP}}_{\phi}(s)] - \text{Var}[V^{\text{ori}}_{\phi}(s)] > 0$. thus:
\begin{equation}
    \text{Var}[A^{\text{RRP}}(s, a)] > \text{Var}[A^{\text{ori}}(s, a)].
\end{equation}
By Lemma~\ref{lem:noise-model-variance}, the policy function variances under two training scenarios satisfy:
\begin{equation}
    \mathcal{V}(\pi^{\text{RRP}}_\theta) > \mathcal{V}(\pi^{\text{ori}}_\theta).
\end{equation}

Similar to the proof of Theorem~\ref{thm:rrp-dqn}, the mapping function $g^k: \mathcal{S} \rightarrow \mathcal{S}$ is defined as:
\begin{equation}
    g^k(s) = \pi^k_\theta(a|s)T(s'|s,a), \quad k \in \{\text{RRP}, \text{ori}\},
\end{equation}
where $T(s'|s,a)$ is the transition function. Thereby, the trajectories for two scenarios are generated by sampling from the state by state:
\begin{equation}
    s^k_h = g^k(s^k_{h-1}), \quad s^k_0 \sim p(s_0), \quad h \in \{1,2,\ldots,H\}, \quad k \in \{\text{RRP}, \text{ori}\}.
\end{equation}
By Lemma~\ref{lem:var-trajectory}, the trajectory variances from the two scenarios satisfy:
\begin{equation}
    \mathcal{V}(\mathcal{T}^{\text{RRP}}) > \mathcal{V}(\mathcal{T}^{\text{ori}}),
\end{equation}
which concludes the proof of Theorem~\ref{thm:rrp-a2c}.

\section{Additional Experimental Results}
\label{app:add-experiments}

\subsection{Additional Results on Ablation Study of Noise Scale}

Table~\ref{app:tab-scale} summarizes the test performance of the final models for both RRP-SAC and RRP-PPO across all tasks, under both sparse- and dense-reward scenarios.

\begin{table*}[h]
\centering
\setlength{\tabcolsep}{1.2pt}
\small
\caption{Average returns obtained from $100$ tests of the final models in the ablation study of noise scale.}
\label{app:tab-scale}
\begin{tabular}{cccccccc}
    \toprule
    \multirow{2}{*}{Rewards} & \multirow{2}{*}{Tasks} & \multicolumn{3}{c}{RRP-SAC} & \multicolumn{3}{c}{RRP-PPO} \\ \cmidrule(lr){3-8}
    & & $\sigma_0^2 = 0.5$ & $\sigma_0^2 = 1.0$ (default) & $\sigma_0^2 = 1.5$ & $\sigma_0^2 = 0.5$ & $\sigma_0^2 = 1.0$ (default) & $\sigma_0^2 = 1.5$ \\
    \midrule
    \multirow{9}{*}{\makecell{Sparse\\Rewards}} 
	& \textit{AntFar} & 109.32$\pm$0.00 & 116.85$\pm$5.32 & 99.51$\pm$0.00 & 89.34$\pm$0.00 & 94.96$\pm$0.00 & 98.42$\pm$0.00 \\ 
    & \textit{HumanStand} & 102.82$\pm$2.00 & 104.34$\pm$0.00 & 103.94$\pm$0.00 & 65.79$\pm$0.00 & 97.05$\pm$0.00 & 94.57$\pm$2.10 \\ 
    & \textit{CheetahFar} & 128.61$\pm$0.00 & 124.01$\pm$4.79 & 136.10$\pm$0.00 &  106.23$\pm$0.00 & 101.35$\pm$1.46 & 119.60$\pm$0.00 \\ 
    & \textit{RobotPush} & 72.29$\pm$0.73 & 75.64$\pm$0.54 & 74.61$\pm$0.00 & 46.09$\pm$0.00 & 40.33$\pm$0.54 & 63.39$\pm$0.00 \\ 
    & \textit{RobotSlide} & 11.29$\pm$0.00 & 15.32$\pm$0.00 & 8.45$\pm$0.00 & 5.04$\pm$0.00 & 17.76$\pm$0.00 & 9.69$\pm$0.00 \\ 
    & \textit{RobotPick} & 49.37$\pm$0.83 & 48.47$\pm$0.59 & 45.60$\pm$1.88 & 31.45$\pm$0.00 & 45.37$\pm$0.00 & 35.38$\pm$0.00 \\ 
    & \textit{HandBlock} & 67.96$\pm$1.70 & 68.33$\pm$1.00 & 60.13$\pm$2.43 & 44.31$\pm$0.00 & 60.32$\pm$0.00 & 57.30$\pm$0.00 \\ 
    & \textit{HandEgg} & 77.21$\pm$0.00 & 74.17$\pm$1.07 & 78.07$\pm$0.00 & 74.33$\pm$0.00 & 79.26$\pm$1.31 & 66.58$\pm$0.63 \\ 
    & \textit{HandPen} & 39.21$\pm$0.49 & 44.41$\pm$0.00 & 44.06$\pm$0.00 & 41.58$\pm$0.00 & 47.03$\pm$1.81 & 44.51$\pm$1.50 \\ 
    \midrule
    \multirow{9}{*}{\makecell{Dense\\Rewards}} 
	& \textit{AntFar} & 593.33$\pm$1.17 & 637.89$\pm$0.00 & 643.31$\pm$0.00 & 573.49$\pm$0.00 & 713.12$\pm$0.00 & 585.97$\pm$0.00 \\ 
    & \textit{HumanStand} & 144.09$\pm$0.00 & 149.72$\pm$4.91 & 143.33$\pm$1.92 & 129.64$\pm$2.72 & 147.42$\pm$0.83 & 148.46$\pm$1.25 \\ 
    & \textit{CheetahFar} & 459.83$\pm$inf & 583.34$\pm$0.00 & 494.43$\pm$0.00 & 452.58$\pm$0.00 & 487.31$\pm$0.00 & 504.04$\pm$0.00 \\ 
    & \textit{RobotPush} & 20.23$\pm$0.00 & 18.11$\pm$0.76 & 17.43$\pm$0.00 & 16.57$\pm$0.11 & 17.37$\pm$1.02 & 17.41$\pm$0.51 \\ 
    & \textit{RobotSlide} & 16.08$\pm$0.00 & 22.57$\pm$0.00 & 19.84$\pm$0.00 & 20.36$\pm$0.00 & 26.88$\pm$0.67 & 23.74$\pm$0.00 \\ 
    & \textit{RobotPick} & 19.29$\pm$0.00 & 19.01$\pm$0.00 & 23.05$\pm$0.00 & 14.34$\pm$0.00 & 21.90$\pm$0.00 & 15.95$\pm$0.00 \\ 
    & \textit{HandBlock} & 77.40$\pm$0.39 & 80.03$\pm$0.00 & 75.50$\pm$0.03 & 70.88$\pm$0.70 & 81.97$\pm$0.00 & 83.25$\pm$1.78 \\ 
    & \textit{HandEgg} & 46.76$\pm$inf & 71.56$\pm$0.00 & 61.54$\pm$0.00 & 69.89$\pm$0.00 & 69.20$\pm$0.00 & 72.57$\pm$0.00 \\ 
    & \textit{HandPen} & 145.67$\pm$0.00 & 123.81$\pm$0.59 & 137.58$\pm$0.00 & 128.66$\pm$0.00 & 130.09$\pm$0.43 & 118.30$\pm$0.00 \\
    \bottomrule
\end{tabular}
\end{table*}

\subsection{Additional Results on Ablation Study of Noise Decay Period}

Table~\ref{app:tab-period} shows the test performance of the final models for both RRP-SAC and RRP-PPO across all tasks, under both sparse- and dense-reward scenarios.

\begin{table*}[t]
\centering
\setlength{\tabcolsep}{1.2pt}
\small
\caption{Average returns obtained from $100$ tests of the final models in the ablation study of noise decay period.}
\label{app:tab-period}
\begin{tabular}{cccccccc}
    \toprule
    \multirow{2}{*}{Rewards} & \multirow{2}{*}{Tasks} & \multicolumn{3}{c}{RRP-SAC} & \multicolumn{3}{c}{RRP-PPO} \\ \cmidrule(lr){3-8}
    & & $\lambda = 0.1$ & $\lambda = 0.3$ (default) & $\lambda = 0.5$ & $\lambda = 0.1$ & $\lambda = 0.3$ (default) & $\lambda = 0.5$ \\
    \midrule
    \multirow{9}{*}{\makecell{Sparse\\Rewards}}
    & \textit{AntFar} & 107.84$\pm$2.69 & 116.85$\pm$5.32 & 117.96$\pm$0.00 & 100.75$\pm$0.00 & 94.96$\pm$0.00 & 101.03$\pm$0.00 \\ 
    & \textit{HumanStand} & 99.64$\pm$0.00 & 104.34$\pm$0.00 & 102.25$\pm$0.00 & 84.95$\pm$1.62 & 97.05$\pm$0.00 & 80.90$\pm$5.85 \\ 
    & \textit{CheetahFar} & 125.61$\pm$0.00 & 124.01$\pm$4.79 & 127.50$\pm$3.51 & 122.03$\pm$0.00 & 101.35$\pm$1.46 & 103.05$\pm$0.00 \\ 
    & \textit{RobotPush} & 69.35$\pm$0.87 & 75.64$\pm$0.54 & 72.89$\pm$1.61 & 28.16$\pm$0.00 & 40.33$\pm$0.54 & 52.55$\pm$0.00  \\ 
    & \textit{RobotSlide} & 9.27$\pm$0.00 & 5.32$\pm$0.00 & 8.79$\pm$0.00 & 14.45$\pm$0.00 & 17.76$\pm$0.00 & 13.52$\pm$0.00 \\ 
    & \textit{RobotPick} & 44.47$\pm$0.03 & 48.47$\pm$0.59 & 41.96$\pm$0.03 & 37.85$\pm$0.00 & 45.37$\pm$0.00 & 40.47$\pm$0.00 \\ 
    & \textit{HandBlock} & 73.21$\pm$1.12 & 68.33$\pm$1.00 & 67.36$\pm$3.05 & 64.21$\pm$0.00 & 60.32$\pm$0.00 & 61.90$\pm$0.00 \\ 
    & \textit{HandEgg} & 69.89$\pm$0.92 & 74.17$\pm$1.07 & 74.35$\pm$0.00 & 73.14$\pm$0.00 & 79.26$\pm$1.31 & 68.23$\pm$2.61 \\ 
    & \textit{HandPen} & 49.82$\pm$0.00 & 44.41$\pm$0.00 & 43.37$\pm$1.96 & 49.15$\pm$0.09 & 47.03$\pm$1.81 & 43.54$\pm$0.00 \\ 
    \midrule
    \multirow{9}{*}{\makecell{Dense\\Rewards}} 
	& \textit{AntFar} & 563.68$\pm$0.00 & 637.89$\pm$0.00 & 651.50$\pm$0.00 & 616.11$\pm$0.00 & 713.12$\pm$0.00 & 627.68$\pm$0.00 \\ 
    & \textit{HumanStand} & 137.39$\pm$4.02 & 149.72$\pm$4.91 & 143.39$\pm$0.00 & 140.44$\pm$0.00 & 147.42$\pm$0.83 & 139.97$\pm$1.03 \\ 
    & \textit{CheetahFar} & 547.22$\pm$0.00 & 583.34$\pm$0.00 & 611.60$\pm$0.00 & 459.57$\pm$0.00 & 487.31$\pm$0.00 & 459.30$\pm$0.00 \\ 
    & \textit{RobotPush} & 17.58$\pm$0.10 & 18.11$\pm$0.76 & 17.99$\pm$0.08 & 16.11$\pm$0.43 & 17.37$\pm$1.02 & 15.13$\pm$1.22 \\ 
    & \textit{RobotSlide} & 26.46$\pm$0.00 & 22.57$\pm$0.00 & 17.75$\pm$0.00 & 15.17$\pm$0.00 & 26.88$\pm$0.67 & 24.77$\pm$0.00 \\ 
    & \textit{RobotPick} & 11.28$\pm$inf & 19.01$\pm$0.00 & 20.50$\pm$0.47 & 22.49$\pm$0.00 & 21.90$\pm$0.00 & 21.18$\pm$0.00 \\ 
    & \textit{HandBlock} & 84.38$\pm$1.01 & 80.03$\pm$0.00 & 87.53$\pm$1.50 & 80.74$\pm$0.00 & 81.97$\pm$0.00 & 76.18$\pm$0.58 \\ 
    & \textit{HandEgg} & 72.05$\pm$0.00 & 71.56$\pm$0.00 & 52.68$\pm$0.00 & 62.92$\pm$0.00 & 69.20$\pm$0.00 & 69.08$\pm$0.00 \\ 
    & \textit{HandPen} & 123.03$\pm$0.00 & 123.81$\pm$0.59 & 127.28$\pm$0.00 & 104.73$\pm$inf & 130.09$\pm$0.43 & 120.13$\pm$0.00 \\ 	
    \bottomrule
\end{tabular}
\end{table*}

\section{Algorithms Implementation Details}
\label{app:implementation}

This section provides the implementation details for the proposed algorithms, including the hyperparameters and neural network structures used in the experiments. The hyperparameters for RRP-SAC and RRP-PPO are summarized in Table~\ref{app:tab-hyperparameters}. 

\begin{table*}[h!]
\centering
\caption{The hyperparameters of RRP-SAC and RRP-PPO in experiments.}
\label{app:tab-hyperparameters}

\setlength{\tabcolsep}{12pt}
\begin{tabular}{ccc}
    \toprule
    Algorithm & Hyperparameters & Values \\
    \midrule
    \multirow{11}{*}{RRP-SAC} & auto-entropy tuning & True \\
    & discounted factor $\gamma$ & 0.99 \\
    & replay buffer size $|\mathcal{D}|$ & $1 \times 10^6$ \\
    & batch size $B$ & 256 \\
    & policy learning rate & $3 \times 10^{-4}$ \\
    & Q-function learning rate & $1 \times 10^{-3}$ \\
    & entropy coefficient $\alpha$ learning rate & $1 \times 10^{-4}$ \\
    & policy networks update frequency (steps) & $2$ \\
    & target networks update frequency (steps) & $1$ \\
    & target networks soft update weight $\tau$ & $5 \times 10^{-3}$ \\
    & burn-in steps & $5000$ \\
    \midrule
    \multirow{11}{*}{RRP-PPO} & number of parallel environments & 4 \\ 
    & discounted factor $\gamma$ & 0.99 \\
    & generalized advantage estimation (GAE) coefficient & 0.95 \\
    & rollout length & $2048$ \\ 
    & number of mini-batches & 32 \\
    & number of update epochs & 10 \\
    & policy learning rate & $3 \times 10^{-4}$ \\
    & value function learning rate & $3 \times 10^{-4}$ \\
    & annealing learning rate & True \\
    & clip coefficient & 0.2 \\
    & value coefficient & 0.5 \\
    \bottomrule
\end{tabular}
\end{table*}

The neural network architectures used in the policy, value function, and Q-function, are detailed in Figure~\ref{app:fig-networks}.

\begin{figure*}[h!]
    \centering
    \includegraphics[width=0.9\linewidth]{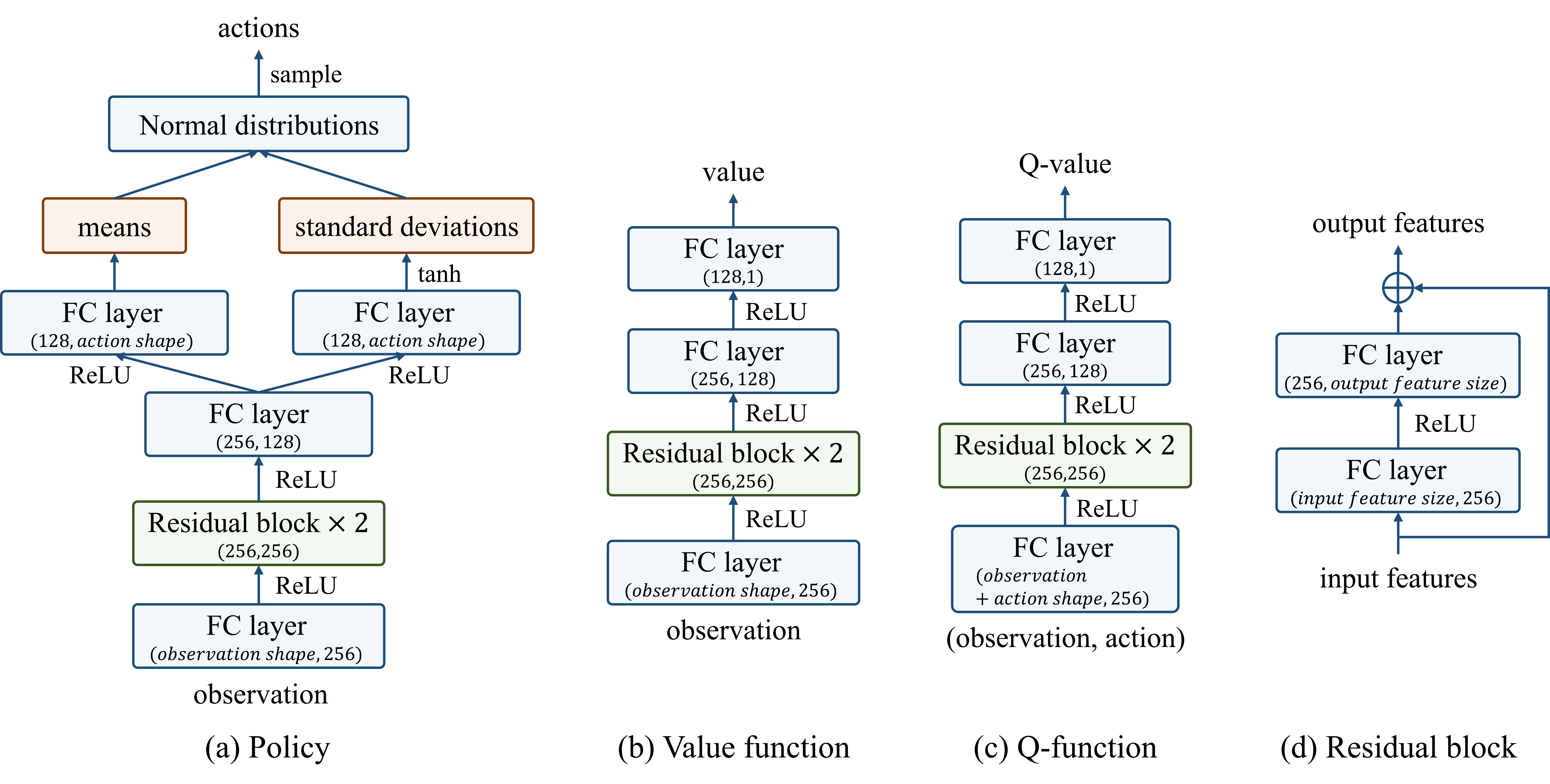}
    \caption{Network structures for the policy, value function, and Q-function in RRP-SAC and RRP-PPO.}
    \label{app:fig-networks}
\end{figure*}

\section{Details of Environments in Experiments}
\label{app:tasks}

The environmental reward structures for the tasks in this paper are detailed in Table~\ref{app:tab-tasks}.

\begin{table}[h!]
\centering
\small
\caption{Environmental reward structures for all tasks, under both sparse and dense reward settings.}
\label{app:tab-tasks}

\begin{tabular}{lp{6.2cm}p{6.2cm}}
    \toprule
    \multirow{2}{*}{Tasks} & \multicolumn{2}{c}{Environmental Reward Structures} \\ \cmidrule(lr){2-3}
	& \makecell[c]{Sparse Rewards} & \makecell[c]{Dense Rewards} \\ 
    \midrule
	\textit{AntFar} & The quadruped robot receives a reward of $1$ if it moves beyond a threshold distance of $5$ units in the x-axis direction. Otherwise, recieves $0$. 
    & The reward is proportional to the robot's displacement projected onto the x-axis, encouraging forward movement. \\ \hline
    \textit{HumanStand} & The humanoid robot earns a reward of $1$ when its center of mass exceeds a height of $0.5$ units, representing a standing pose. Otherwise, the reward is $0$. 
    & The reward directly corresponds to the height of the robot's center of mass, providing continuous feedback during the transition from lying down to standing. \\ \hline
    \textit{CheetahFar} & The half-cheetah robot is rewarded with $1$ when it surpasses a threshold of $5$ units in the x-axis direction. Rewards remain $0$ if the threshold is not crossed. 
    & The reward is proportional to the forward distance traveled by the half-cheetah robot along the x-axis, promoting forward locomotion. \\ \hline
    \textit{RobotPush} & The robotic arm receives a reward of $1$ if the block is moved to within $0.05$ units of the target position. Otherwise, the reward is $0$. 
    & The reward is calculated as $-d + d_0$, where $d$ is the current distance between the block and the target position, and $d_0$ is the initial distance, encouraging closer positioning of the block. \\ \hline
    \textit{RobotSlide} & The robotic arm earns a reward of $1$ for sliding the block to within $0.05$ units of the target location on a frictionless surface. Otherwise, the reward remains $0$. 
    & The reward is similarly defined as $-d + d_0$, but the low-friction surface enables the block to slide more easily with smaller forces. \\ \hline
    \textit{RobotPick} & The robotic arm is rewarded with $1$ for successfully picking up and placing the block within $0.05$ units of the target position. If unsuccessful, the reward is $0$. 
    & The reward follows the formula $-d + d_0$, where $d$ is the distance between the block and the target. However, the task requires the arm to first lift the block and then move it to the desired position. \\ \hline
    \textit{HandBlock} & The robotic hand receives a reward of $1$ if the block's rotation around the z-axis is within $0.1$ radians and the Euclidean distance to the target position is below $0.01$ units. Otherwise, the reward is $0$. 
    & The reward is calculated as the negative sum of the Euclidean distance to the target position and the angular difference from the target orientation, with an initial offset added for normalization. \\ \hline
    \textit{HandEgg} & The robotic hand earns a reward of $1$ for successfully rotating an egg-shaped object around its z-axis to the target orientation, with a positional error below $0.01$ units and an angular error under $0.1$ radians. Otherwise, the reward remains $0$. 
    & The reward is determined by the negative sum of the Euclidean distance to the target position and the angular discrepancy from the target orientation, similar to the \textit{HandBlock} task but applied to an egg-shaped object. \\ \hline
    \textit{HandPen} & The robotic hand is rewarded with $1$ when it rotates a pen-shaped object to the target orientation, with the position and orientation errors meeting the same thresholds as in \textit{HandBlock}. Otherwise, the reward is $0$. 
    & The reward function mirrors that of the \textit{HandBlock} task, combining the negative Euclidean distance to the target and the angular discrepancy, but is applied to a pen-shaped object. \\ 
    \bottomrule
\end{tabular}
\end{table}

\end{document}